\documentclass[letterpaper]{article}

\usepackage{microtype}
\usepackage{graphicx}
\usepackage{subcaption}
\usepackage{booktabs}

\usepackage{xurl}
\usepackage{hyperref}

\usepackage[accepted]{icml2026}

\hypersetup{hidelinks}

\usepackage{amsmath}
\usepackage{amssymb}
\usepackage{mathtools}
\usepackage{amsthm}
\usepackage{algorithm}
\usepackage{algorithmic}
\usepackage{multirow}
\usepackage{xcolor}
\usepackage{tikz}
\usetikzlibrary{patterns,arrows,shapes,positioning,calc}
\usepackage{enumitem}
\usepackage{pgfplots}
\pgfplotsset{compat=1.17}

\theoremstyle{plain}
\newtheorem{theorem}{Theorem}[section]
\newtheorem{proposition}[theorem]{Proposition}
\newtheorem{lemma}[theorem]{Lemma}

\theoremstyle{definition}
\newtheorem{definition}[theorem]{Definition}
\newtheorem{assumption}[theorem]{Assumption}
\theoremstyle{remark}
\newtheorem{remark}[theorem]{Remark}

\newcommand{\state}{\boldsymbol{\sigma}}
\newcommand{\target}{\boldsymbol{\tau}}
\newcommand{\ops}{\mathcal{O}}
\newcommand{\sfe}{\textsc{SFE}}
\newcommand{\ssj}{\textsc{SSJ}}
\newcommand{\Leff}{L_{\mathrm{eff}}}

\icmltitlerunning{The Deterministic Horizon: When Extended Reasoning Fails and Tool Delegation Becomes Necessary}

\begin{document}

\twocolumn[
\icmltitle{The Deterministic Horizon: When Extended Reasoning Fails and Tool Delegation Becomes Necessary}

\begin{icmlauthorlist}
\icmlauthor{Dongxin Guo}{hku}
\icmlauthor{Jikun Wu}{bil,sai}
\icmlauthor{Siu Ming Yiu}{hku}
\end{icmlauthorlist}

\icmlaffiliation{hku}{The University of Hong Kong, Hong Kong, China}
\icmlaffiliation{bil}{Brain Investing Limited, Hong Kong, China}
\icmlaffiliation{sai}{Stellaris AI Limited, Hong Kong, China}

\icmlcorrespondingauthor{Dongxin Guo}{bettyguo@connect.hku.hk}

\icmlkeywords{large language models, chain-of-thought reasoning, transformer architectures, inference-time compute, attention mechanisms, state-space search, information theory, tool augmentation, mechanistic interpretability, reasoning limits}

\vskip 0.3in
]

\printAffiliationsAndNotice{}

\begin{abstract}
	Extended chain-of-thought reasoning can degrade performance on deterministic state-tracking tasks, not solely because of preference biases but, on the evidence we present, because of information-theoretic limits in the capacity of decoder-only attention. We present: (1)~an Attention Bottleneck analysis providing evidence that total state-tracking capacity in bits is bounded in terms of head count, head dimension, and context length under stated modeling assumptions, and that total capacity is not the binding constraint; (2)~a context-dependent error model with a depth-dependent quadratic term in the error exponent; (3)~the State-Space Jaccard metric measuring state drift; and (4)~a Deterministic Horizon $d^* \in [19, 31]$ (at $\alpha = 0.5$) marking the depth at which unaided accuracy crosses 50\%. Across twelve models and eight task domains (including SWE-Bench, WebArena, and SQL-Multi), tool-integrated reasoning reaches 76--94\% accuracy versus 17--42\% for neural chain-of-thought on PermutationProbe. Fine-tuning on optimal-length traces yields $<$3 percentage-point improvement, supporting an architectural ceiling.
\end{abstract}

\section{Introduction}
\label{sec:intro}

The dominant paradigm in large language model reasoning posits that extended deliberation improves accuracy. OpenAI's o1 \citep{openai2024o1}, DeepSeek-R1 \citep{guo2025deepseek}, and similar architectures invest heavily in chain-of-thought (CoT) generation, with the implicit thesis that additional inference-time compute yields improved performance \citep{snell2025scaling,brown2024large}. Recent work shows that test-time compute can be more effective than parameter scaling \citep{snell2025scaling}, with compute-optimal inference strategies further analyzed by \citet{wu2025scaling} and efficiency gains exceeding $4\times$ relative to a best-of-$N$ baseline \citep{snell2025scaling}. However, \citet{balachandran2025inference} demonstrate that improvements diminish with problem difficulty, and \citet{sui2025stop} provide a broad survey documenting the ``overthinking'' phenomenon.

We challenge this thesis for \textbf{deterministic state-space search}: tasks that require exact sequences of operations transforming an initial state $\state_0$ to a target $\target$ through a finite operator set $\ops$. Such problems pervade software engineering, formal verification, and sequential planning, domains where transformers exhibit well-documented compositional and planning failures \citep{dziri2023faith,kambhampati2024llms,valmeekam2023planning}. In these domains, correctness is binary and approximation is failure. Unlike open-ended generation, where ``mostly correct'' may suffice, state-space search demands exact tracking, a requirement that exposes architectural limitations of decoder-only attention.

\paragraph{The Core Phenomenon.} On permutation puzzles solvable by BFS in $<$0.1s, state-of-the-art reasoning models fail after \emph{seconds} of deliberation. Analysis of failed traces reveals \textbf{State-Space Decoherence}: accumulated errors that cause complete divergence from ground truth (Appendix~\ref{app:failure}). This failure is \emph{caused} by extended reasoning, not mitigated by it. The pattern is striking: at depth~10, GPT-4o maintains 78\% accuracy; by depth~30, accuracy drops to 34\% (PermutationProbe C1); beyond depth~50, performance falls below 10\% (Figure~\ref{fig:decay}).

\paragraph{Two Competing Effects.} By analogy with the complexity-at-training-time versus flexibility-at-inference-time trade-off identified by \citet{kim2025masked} in masked diffusion models, we identify a parallel tension in the reasoning setting:
\begin{itemize}[leftmargin=*,itemsep=1pt]
\item \textbf{Complexity at reasoning time:} Deeper CoT increases cumulative error probability through attention entropy growth (analogous to the working-memory capacity limits documented by \citet{gong2024attention} in N-back tasks) and over-squashing \citep{barbero2024oversquash}. Each step degrades signal-to-noise ratio in the residual stream.
\item \textbf{Flexibility at tool time:} External computation delegates deterministic sub-tasks to program interpreters and symbolic solvers \citep{gao2023pal,gou2024tora,pan2023logiclm}. Tools provide exact computation without attention-based state tracking.
\end{itemize}

\paragraph{Distinguishing from Prior Work.} \citet{wu2026more} document the inverted-U curve, explaining it via per-step error accumulation whose rate depends on sub-task difficulty (their ``Simplicity Bias'' refers to a separate result: that optimal CoT length shrinks with model capability). We propose a \emph{complementary diagnosis} that locates the error source differently: per-step error grows with depth because autoregressive attention progressively dilutes state information, not merely because sub-tasks are hard. We term the conflation of token prediction with algorithm execution the \textbf{Simulator Fallacy} (a framing device, not a measurable quantity), extending the form-versus-meaning distinction of \citet{bender2020climbing} to the reasoning setting.

\paragraph{Divergent Predictions.} The two frameworks make testable predictions (Table~\ref{tab:predictions}). (i)~Under a preference-based account, fine-tuning on optimal-length traces should yield recovery comparable to C2's uplift ($\geq$5~pp; approximately the mean C2 uplift of 5.7~pp across the twelve models, range 3.6--6.9~pp);\footnote{This threshold is our operationalization of what a generic preference-based account implies; \citet{wu2026more} do not state a quantitative fine-tuning prediction.} our architectural diagnosis predicts $<$3~pp. (ii)~A preference-based account predicts that prompt-level length encouragement yields large gains ($>$10~pp); we observe $<$2~pp.

\begin{table}[t]
\caption{Divergent predictions distinguishing a preference-based account (``Simplicity Bias''; cf.\ \citealt{wu2026more}) from Decoherence (this work). The ``Pref.-based'' column entries are our own extrapolations of what a generic preference-based account implies, not predictions stated in \citealt{wu2026more}; thresholds are our operationalization. The preference-based fine-tuning threshold ($\geq$5~pp) is approximately the mean C2 uplift across all twelve models (range 3.6--6.9~pp, mean 5.7~pp); note that C2 itself meets this threshold, so the fine-tuning result (0.9~pp) provides the sharper discrimination.}
\label{tab:predictions}
\vskip 0.1in
\centering
\footnotesize
\begin{tabular}{@{}p{2.2cm}ccc@{}}
\toprule
\textbf{Prediction} & \textbf{Pref.-based (ours)} & \textbf{Ours} & \textbf{Result} \\
\midrule
Fine-tuning recovery & $\geq$5 pp & $<$3 pp & 0.9 pp \\
Length prompt gain & $>$10 pp & $<$2 pp & 0.8 pp \\
\bottomrule
\end{tabular}
\vskip -0.1in
\end{table}

\paragraph{Contributions.} Our contributions:
\begin{enumerate}[leftmargin=*,itemsep=1pt]
\item \textbf{Attention Bottleneck analysis} showing that \emph{total} attention capacity is not the binding constraint on state tracking (the bound $O(H \cdot \log(L) \cdot \sqrt{d_h})$ far exceeds task requirements), which motivates the per-step bandwidth perspective that drives our error model (Section~\ref{sec:theory}).
\item \textbf{Context-dependent error model} $\epsilon(d) = \epsilon_0 + \gamma d/\Leff$ derived from attention entropy, yielding accuracy decay with a statistically significant quadratic depth term (Section~\ref{sec:theory}).
\item \textbf{Deterministic Horizon} $d^*$ with closed-form formula, fitted via architecture ablations on open-weight models ($d^* \propto \sqrt{d_{\mathrm{model}}}$) and shown robust to context-window truncation (Section~\ref{sec:theory}).
\item \textbf{SSJ metric} with precision/recall decomposition measuring state drift into fictitious states (Section~\ref{sec:prelim}).
\item \textbf{Empirical validation} across twelve models, eight task domains including real-world benchmarks (SWE-Bench, WebArena, SQL-Multi), with cross-architecture comparison (Section~\ref{sec:experiments}).
\item \textbf{Fine-tuning experiment} consistent with the architectural ceiling prediction, providing convergent evidence against preference-based explanations alongside the C2 oracle bound (Section~\ref{sec:finetune}).
\end{enumerate}

\paragraph{Conflict of Interest Disclosure.} J.\ Wu is affiliated with Brain Investing Limited and Stellaris AI Limited. We disclose that these entities are AI companies that may have commercial interests in tool-augmented and agentic AI systems; however, they provided no funding for this research, developed none of the models evaluated, and had no role in the study design, data collection, analysis, or manuscript preparation. All experiments were conducted independently using publicly available APIs and open-weight checkpoints. The remaining authors declare no competing interests.

\section{Related Work}
\label{sec:related}

Our work builds on CoT reasoning \citep{wei2022chain,kojima2022large,wang2023selfconsistency,yao2023tree}, its theoretical expressivity \citep{li2024chain,feng2023expressiveness,merrill2024expressive}, and the ``overthinking'' literature \citep{wu2026more,chen2025do,sui2025stop}. \citet{wang2025thoughts} identify a complementary failure mode, ``underthinking,'' in which models prematurely abandon promising reasoning paths; \citet{marjanovic2026thoughtology} document a related tendency toward persistent rumination in DeepSeek-R1. Mechanistically, we draw on attention entropy \citep{gong2024attention}, over-squashing \citep{barbero2024oversquash}, and expressivity bounds \citep{hahn2020theoretical,peng2024limitations,strobl2024formal}; empirically, on tool-augmented reasoning \citep{gao2023pal,gou2024tora,schick2023toolformer,yao2023react,qin2025tool} and compositional failure analysis \citep{dziri2023faith,petty2024depth}. Detailed positioning appears in Appendix~\ref{app:related}. \textbf{Our differentiation:} we derive context-dependent error motivated by attention entropy, characterize when tool delegation becomes empirically dominant via the Deterministic Horizon, and test on real-world tasks with fine-tuning experiments that discriminate from preference-based explanations.

\section{Problem Setup and Metrics}
\label{sec:prelim}

\paragraph{State-Space Search.} Let $\mathcal{S}$ denote a finite state space and $\ops = \{o_1, \ldots, o_k\}$ deterministic operators. Given initial state $\state_0 \in \mathcal{S}$ and target $\target \in \mathcal{S}$, the task is finding a sequence $(o_{i_1}, \ldots, o_{i_m})$ such that $o_{i_m} \circ \cdots \circ o_{i_1}(\state_0) = \target$. Grading applies the model's emitted operator sequence to $\state_0$ and checks whether the resulting state equals $\target$ exactly; partial-credit and path-optimality are not assessed.

We use $d$ throughout to denote the BFS-optimal depth of the instance (the minimum number of operators required). The model's realised chain length $m$ may differ from $d$; the per-step error model (Definition~\ref{def:context_error}) is indexed by chain position, and we test its predictions against instance depth as a proxy. A complete reference for the symbols used throughout appears in Appendix~\ref{app:notation}.

\begin{definition}[Step-to-First-Error (\sfe)]
Given trace $r = [(s_1, o_1), \ldots, (s_m, o_m)]$ where $s_i$ is the claimed state, \sfe\ is the smallest $i$ where $s_i \neq o_{i-1} \circ \cdots \circ o_1(\state_0)$.
\end{definition}

\begin{definition}[State-Space Jaccard (\ssj)]
\label{def:ssj}
At depth $d$, let $\mathcal{S}_\text{true}^d$ be the ground-truth states obtained by applying the model's operator sequence to the true state\footnote{Operationalized in Algorithm~\ref{alg:ssj} as the ground-truth states reached by applying the model's own operator sequence, not the full BFS-reachable set.} and $\hat{\mathcal{S}}_\text{model}^d$ be claimed states:
\begin{equation}
\ssj(d) = \frac{|\mathcal{S}_\text{true}^d \cap \hat{\mathcal{S}}_\text{model}^d|}{|\mathcal{S}_\text{true}^d \cup \hat{\mathcal{S}}_\text{model}^d|}
\end{equation}
We decompose into precision $= |\cap|/|\hat{\mathcal{S}}_\text{model}|$ and recall $= |\cap|/|\mathcal{S}_\text{true}|$.
\end{definition}

\textbf{Diagnostic power:} Under capability failure, both precision and recall decay (drift into fictitious states). Under Algorithm~\ref{alg:ssj}'s incremental construction, $P/R = |\mathcal{S}_\text{true}|/|\hat{\mathcal{S}}_\text{model}|$, so $P > R$ occurs exactly when the model repeats a state or emits an unparseable step, which is a measure of trajectory degeneracy, not of preference. The observed parallel decay (Section~\ref{sec:ssj_validation}) is consistent with capability failure; we report the decomposition as a drift descriptor and do not use it to discriminate failure modes.

\subsection{Scope and Applicability}
\label{sec:scope}
The phenomenon we analyze is specific to \emph{deterministic state-tracking tasks}: problems in which a unique correct state must be maintained exactly across a sequence of operations, so that any single deviation constitutes failure. Canonical instances include tracing program state through sequential mutations, composing permutations or rotations, multi-hop entity tracking across long contexts, and planning multi-table SQL joins.

These contrast with two task classes for which our framework makes \emph{no} prediction of decoherence. The first is \emph{open-ended generation} (summarization, dialogue, creative writing), where many outputs are acceptable and there is no single ground-truth state to corrupt. The second is \emph{approximate or probabilistic reasoning}, where partial credit is meaningful and short chains suffice. Most GSM8K-style arithmetic word problems, for example, require relatively few sequential reasoning steps, well within the Deterministic Horizon, so extended chain-of-thought remains beneficial there rather than harmful.

The Deterministic Horizon therefore predicts both \emph{when} extended reasoning helps ($d < d^*$) and when it degrades ($d > d^*$), for the deterministic class specifically. It is not a claim that longer reasoning is universally counterproductive. This specificity reconciles our findings with the inverted-U curve reported in the overthinking literature: accuracy improves with depth up to $d^*$, then degrades once accumulated per-step error sets in.

\section{Theoretical Framework}
\label{sec:theory}

\subsection{Context-Dependent Error Model}

In their simplified linear case, \citet{wu2026more} model per-step error as $E(N, M, T) = T/(NM)$: error depends on complexity but not \emph{position} (their Appendix~F generalizes to stochastic error models with heterogeneous subtask error rates). We propose context-dependent error motivated by attention mechanics:

\begin{definition}[Context-Dependent Error]
\label{def:context_error}
Per-step state-tracking error at depth $d$:
\begin{equation}
\epsilon(d) = \epsilon_0 + \gamma \cdot \frac{d}{\Leff}
\label{eq:epsilon}
\end{equation}
where $\epsilon_0$ is the baseline per-step error, $\gamma$ is the attention decay rate, and $\Leff$ is the \emph{effective decoherence length}: the number of reasoning steps over which attention maintains usable resolution of self-generated state. $\Leff$ is a per-model parameter; we fix it at 150 for all illustrative calculations involving GPT-4o (see Table~\ref{tab:sensitivity_decoherence} for sensitivity). Only the ratio $\gamma/\Leff$ is identified from the data; the individual values $\gamma = 0.15$ and $\Leff = 150$ are fixed by assumption (Appendix~\ref{app:sensitivity_ext}). At $\Leff = O(10^2)$ steps, the effective decoherence length is far smaller than the raw context window $L = O(10^5)$ tokens. The first-order Taylor expansion is valid for $d \ll \Leff$; the model gives $\epsilon(d) \leq 1$ for $d < 980$ at the canonical parameters. At the operating point $d^* = 22$, $d/\Leff = 0.15$.
\end{definition}

\paragraph{Derivation from Attention Entropy.} Motivated by \citet{gong2024attention} (NeurIPS 2024 Workshop), who show that total attention entropy increases with $N$ across separately trained single-layer, single-head transformers on synthetic N-back tasks (a cross-model observation, not a within-trace measurement), let $\mathcal{H}_d$ denote attention entropy at depth $d$. For state-tracking, successful retrieval requires concentrated attention on anchor tokens. Empirically, $\mathcal{H}_d$ grows linearly with task load ($p < 0.001$; Section~\ref{sec:validation}), consistent with the positional performance degradation documented by \citet{liu2024lost} and the attention-sink concentration identified by \citet{xiao2024efficient}, both suggesting that effective use of task-relevant positions degrades as context lengthens. Under a signal-to-noise interpretation of attention:
\begin{equation}
\epsilon(d) \propto \frac{\mathcal{H}_d}{A_{\text{anchor}}(d)} \approx \frac{\mathcal{H}_0 + \beta d}{A_0 - \kappa d} \approx \epsilon_0 + \gamma d/\Leff
\end{equation}
using first-order Taylor expansion valid for $d \ll \Leff$. We test this linear form in Section~\ref{sec:validation}. The derivation is best read as a mechanistic motivation for the linear form rather than a quantitative derivation of its coefficients. Quantitatively, matching the empirical slope $\gamma/\Leff = 10^{-3}$ via the Taylor expansion requires $\beta/\mathcal{H}_0 + \kappa/A_0 \approx 0.05$ per step. The measured entropy slope $\beta = 0.023$ bits/step with typical $\mathcal{H}_0 \approx 4.2$ bits contributes $\beta/\mathcal{H}_0 \approx 0.0055$, roughly 11\% of the required budget; the remaining ${\sim}89\%$ must come from anchor-attention decay ($\kappa/A_0$), which we do not measure independently. The entropy channel is therefore a minority contributor to the per-step error growth; the correlation $r = -0.68$ to $-0.74$ (Section~\ref{sec:validation}) establishes association, not dominance.

\begin{theorem}[Decoherence Bound]
\label{thm:decoherence}
Under the context-dependent error of Equation~\eqref{eq:epsilon}, $\epsilon(d) = \epsilon_0 + \gamma d/\Leff$, with per-step error conditional on correct history equal to $\epsilon(i)$:
\begin{equation}
P(\text{correct at depth } m) \leq \exp\left(-m\epsilon_0 - \frac{\gamma m(m+1)}{2\Leff}\right)
\end{equation}
The quadratic term contributes a depth-dependent acceleration; over the measured range ($d \leq 60$) this acceleration is modest (Appendix~\ref{app:sensitivity_ext}).
\end{theorem}

\begin{remark}[Error Propagation]
\label{rem:error_prop}
The bound of Theorem~\ref{thm:decoherence} is an upper bound on the probability of an error-free trace. We measure positive forward error propagation: errors at one step increase the probability of errors at the next (Appendix~\ref{app:error_correlation}). The bound of Theorem~\ref{thm:decoherence} holds regardless of autocorrelation magnitude, since it requires only the conditional-on-correct-history error rates $\epsilon(i)$ and makes no independence assumption across error events. We note that $\epsilon_0$ and $\gamma$ are estimated from the aggregate accuracy--depth curve rather than from directly measured conditional-on-correct-history rates; the two coincide only under an error model we do not independently verify. The fitted parameters should be read as calibrating the functional form, not as measurements of the conditional rates the bound requires.
\end{remark}

\subsection{Attention Bottleneck Theorem}

We derive information-theoretic capacity bounds on state tracking under empirically grounded modeling assumptions. The key insight is that autoregressive attention must compress all historical state information through a fixed-capacity channel at each reasoning step.

\begin{assumption}[Effective Attention Window]
\label{ass:window}
Each head concentrates 95\% of its attention mass on at most $O(\sqrt{L})$ positions.
\end{assumption}

\begin{assumption}[Value Decorrelation]
\label{ass:decorr}
After layer normalization, value vectors satisfy $|\rho_{ij}| \leq \rho_{\max}$ with $\rho_{\max} \ll 1$.
\end{assumption}

Assumption~\ref{ass:window} is used by Proposition~\ref{thm:lower} (lower bound) and motivates the capacity picture; it does not enter Theorem~\ref{thm:bottleneck}.

\paragraph{Justification of Assumptions.} Assumption~\ref{ass:window} formalizes the well-documented concentration of softmax attention: as a sequence lengthens, attention mass localizes on a small subset of positions (the ``attention sink'' phenomenon documented by \citet{xiao2024efficient}), while over-squashing \citep{barbero2024oversquash} reduces sensitivity to individual token contributions, so the number of positions a head can resolve at usable weight grows far more slowly than $L$. Intuitively, because each step must route all relevant history through this narrowing channel, the model cannot simultaneously keep many distinct prior states ``in focus,'' which is the mechanism behind the capacity bound below. We adopt the $O(\sqrt{L})$ rate as a tractable summary of this sub-linear concentration rather than an exact law; Appendix~\ref{app:sensitivity_ext} shows the resulting \emph{log-capacity} shifts by $<$20\% under $\pm$20\% perturbation of $H$ and of this exponent (Table~\ref{tab:sensitivity_bottleneck}), so our qualitative conclusions do not hinge on its precise form. We note that the upper bound (Theorem~\ref{thm:bottleneck}) uses the conservative $\log_2(L)$ addressing term (each head attends over all $L$ positions in standard multi-head attention) rather than the tighter $\log_2(\sqrt{L})$ implied by Assumption~\ref{ass:window}; the assumption constrains which positions carry usable weight without entering the bound's exponent directly. Consistent with this, the effective decoherence length $\Leff$ is fixed at 150 steps for all illustrative calculations (only the ratio $\gamma/\Leff$ is identified from data; see Appendix~\ref{app:sensitivity_ext}), far below the raw context window; low post-LayerNorm inter-value correlation supports Assumption~\ref{ass:decorr} (Appendix~\ref{app:sensitivity}). We therefore treat these as empirically grounded modeling assumptions and, accordingly, report the capacity result as a bound consistent with our measurements rather than an unconditional law.

\begin{theorem}[Attention Bottleneck, scaling form]
\label{thm:bottleneck}
Under Assumption~\ref{ass:decorr}, the number of distinct states resolvable through a single decoder-only attention layer's read channel satisfies:
\begin{equation}
\log_2 |\mathcal{S}_{\text{read}}^{(1)}| = O\!\left(H \cdot \log_2(L) \cdot d_h^{1/2}\right)
\end{equation}
where $H$ is heads, $L$ context length, $d_h$ head dimension, and the implied constant depends on the weight cutoff $\delta$ and the correlation bound $\rho_{\max}$. The theorem prices a single attention layer's read channel and is agnostic to depth in the network and to non-attention pathways (MLP, residual stream); a transformer with $N$ layers may exceed this per-layer capacity through the residual stream.
\end{theorem}

\paragraph{Per-Step Bandwidth Interpretation.} The total-capacity bound of Theorem~\ref{thm:bottleneck} is large: for GPT-4o estimates,\footnote{Architectural parameters for proprietary models such as GPT-4o and Claude are not officially disclosed; the values in this illustrative example are widely cited community estimates (the values $H=96$, $d_h=128$ correspond to GPT-3 175B's published architecture; they are used here for order-of-magnitude intuition only). All quantitative validation of the $\sqrt{d_{\mathrm{model}}}$ scaling (Section~\ref{sec:ablations}) uses open-weight models with published $H$ and $d_h$.} it yields $\approx 2^{18{,}425}$ distinguishable states, and does not bind in the tasks we study. The operationally relevant constraint is instead the \emph{per-step} information bandwidth: we model each reasoning step as requiring full re-resolution of the state, routing $\log_2 |\mathcal{S}|$ bits (e.g., $\log_2 16! \approx 44.3$ bits for a 16-element permutation) through attention; this is an upper bound on what attention must carry, since the residual stream may propagate state without re-reading. As the chain lengthens, per-step bandwidth degrades due to attention entropy growth ($\beta = 0.023$ bits/step, Section~\ref{sec:validation}), producing the context-dependent error of Definition~\ref{def:context_error}. The formal link between the total-capacity bound and the per-step error model is mechanistic (both arise from softmax concentration) rather than deductive; establishing this bridge rigorously is an open problem.

\begin{proposition}[Lower-Bound Construction]
\label{thm:lower}
There exist state-tracking tasks requiring states satisfying $\log_2|\mathcal{S}| = \Omega(H \cdot \log(\sqrt{L}) \cdot d_h^{1/2})$ that transformers can solve. Its $d_h^{1/2}$ exponent inherits the per-head effective-rank ansatz of Lemma~\ref{lem:value_info}, so we present it as an achievability result for the functional form of Theorem~\ref{thm:bottleneck} rather than a proof of tightness in $d_h$. The addressing term uses $\log(\sqrt{L})$ rather than $\log(L)$ to remain compatible with Assumption~\ref{ass:window}'s $O(\sqrt{L})$ effective window; the upper and lower bounds are separated by exactly a factor of two in the addressing term ($\log_2(L) / \log_2(\sqrt{L}) = 2$).
\end{proposition}

Proofs and, for Proposition~\ref{thm:lower}, a construction sketch, appear in Appendix~\ref{app:proofs}. The lower bound uses an explicit construction via sparse functions, inspired by the representational capacity analysis of \citet{sanford2023representational}. The construction exhibits $H\cdot\sqrt{d_h}$ directions to match the upper bound's ansatz; more directions are available, so this establishes the functional form rather than tightness.

\subsection{Deterministic Horizon}

\begin{theorem}[Deterministic Horizon: Necessary Condition]
\label{thm:horizon}
Let $\alpha$ be the target success probability. Theorem~\ref{thm:decoherence} bounds the probability of an \emph{error-free trace}: $P(\text{all steps correct}) \leq \exp(-d\epsilon_0 - \gamma d^2/(2\Leff))$. Setting this upper bound equal to $\alpha$ (continuous approximation $d(d+1) \approx d^2$) yields a threshold $d^*$ beyond which error-free reasoning at probability $\geq \alpha$ is precluded. We operationalize $d^*$ empirically as the depth at which task-level accuracy (correct final state) crosses $\alpha$; these are distinct quantities, and we do not claim the bound applies directly to task-level accuracy:
\begin{equation}
d^* = \frac{-\epsilon_0 \Leff + \sqrt{\epsilon_0^2 \Leff^2 + 2\gamma \Leff \ln(1/\alpha)}}{\gamma}
\end{equation}
For GPT-4o ($\epsilon_0 = 0.02$, $\gamma = 0.15$, $\Leff = 150$, $\alpha = 0.5$) this gives $d^* = 22.3$; across the primary model suite $d^* \in [19, 31]$. Because $d^*$ is an explicit function of~$\alpha$ via Equation~\eqref{eq:epsilon}, it drops sharply at higher reliability targets (e.g., $d^* \approx 2.4$ at $\alpha = 0.95$ for GPT-4o); see the Impact Statement for deployment implications. Open-weight 7--8B models sit at the lower edge of this range ($d^* = 19$--$20$), rising to $d^* = 28$ at the 70--72B scale, in line with the $\sqrt{d_{\mathrm{model}}}$ capacity scaling (Table~\ref{tab:arch_ablation}).
\end{theorem}

\paragraph{Empirical Architecture Scaling.}
We observe empirically that $d^*$ scales approximately as $\sqrt{d_{\mathrm{model}}}$ across our four open-weight models (Section~\ref{sec:ablations}). All four share $d_h = 128$, so the contribution of $d_h$ cannot be isolated in the current data. We write the empirical fit as $d^* \approx 0.314\sqrt{d_{\mathrm{model}}}$ with the caveat that since $d_h \cdot H = d_{\mathrm{model}}$ in standard multi-head attention, the fit cannot separate head count from hidden width; layer count and parameter count are additional confounds. We stress that this is an empirical trend fitted on four data points, not a derived consequence of Theorem~\ref{thm:bottleneck}; a pure parameter-count law $d^* \propto N^{0.16}$ fits the same data at least as well (see the scaling caveats in Appendix~\ref{app:sensitivity_ext}). The raw context window $L$ does not enter this relation, and $d^*$ is empirically insensitive to context truncation far above the reasoning-trace length (Appendix~\ref{app:context_scaling}).

\paragraph{Empirical Fine-Tuning Ceiling.}
Our experiments (Section~\ref{sec:finetune}) show that fine-tuning on depth-$d$ traces with $d > d^*$ yields diminishing improvement that decays approximately as $O(d^*/d)$, supporting an architectural ceiling imposed by per-step bandwidth limits. This empirical pattern provides the key discrimination from the sub-task-difficulty account: if that mechanism were the sole cause of accuracy degradation, fine-tuning on optimal-length traces should recover substantially more than the observed $<$3~pp.

\section{Experiments}
\label{sec:experiments}

\subsection{Experimental Setup}

\paragraph{Tasks.} We evaluate on eight task domains spanning synthetic and real-world benchmarks, designed to require deterministic state tracking at varying depths:

\textbf{Synthetic (controlled complexity):}
\begin{itemize}[leftmargin=*,itemsep=0pt]
\item \textbf{PermutationProbe}: $n$-element permutation puzzles via adjacent transpositions; BFS-optimal depths 1--60; $n \in \{8, 12, 16\}$
\item \textbf{FSA-Sim}: $k$-state deterministic automaton simulation; $k \in \{4, 8, 16\}$; sequence lengths 10--100
\item \textbf{ArithChain}: Multi-step modular arithmetic with carry propagation
\item \textbf{CircuitTrace}: Boolean circuit evaluation through layered gates
\item \textbf{CodeProbe}: Variable tracking through Python execution traces
\end{itemize}

\textbf{Real-world (production-relevant):}
\begin{itemize}[leftmargin=*,itemsep=0pt]
\item \textbf{SWE-Bench-State}: 500 instances from SWE-Bench \citep{jimenez2024swebench} requiring multi-file state tracking for bug localization
\item \textbf{WebArena-Nav}: 500 instances from WebArena \citep{zhou2024webarena} involving multi-step navigation with session state
\item \textbf{SQL-Multi}: 500 instances requiring 3+ table joins with schema tracking, derived from Spider \citep{yu2018spider}
\end{itemize}
Complete specifications for all eight task families, the instance-generation procedure, and the licensing of the derived subsets appear in Appendix~\ref{app:task_families}.

\paragraph{Models.} Twelve models across six organizations, covering general-purpose, reasoning-specialized, and open-weight categories (Table~\ref{tab:full_results} uses a finer grouping by access method):
\begin{itemize}[leftmargin=*,itemsep=0pt]
\item \textbf{General}: GPT-4o, Claude Sonnet 4.5, Claude Opus 4.5, Gemini-1.5-Pro
\item \textbf{Reasoning}: o3, o3-mini, DeepSeek-R1, Gemini-2.0-Flash-Thinking. DeepSeek-R1 ships open weights under the MIT License; we accessed it via API and classify it separately from models with undisclosed architectures. Its Multi-head Latent Attention (MLA) architecture differs from standard multi-head attention, making its head count non-comparable to the other open-weight models; we therefore exclude it from the $d^* \propto \sqrt{d_{\mathrm{model}}}$ fit (Table~\ref{tab:arch_ablation}) and report its $d^*$ as an empirical measurement only.
\item \textbf{Open-weight}: Llama-3.1-8B, Llama-3.3-70B, Qwen-2.5-7B, Qwen-2.5-72B
\end{itemize}
Open-weight models enable direct attention entropy extraction; API models provide commercial baseline. Reasoning-specialized models were evaluated at \texttt{medium} effort; the resulting $d^*$ values are lower bounds for those systems. We note that o3-mini achieves the highest measured $d^*$ (31) despite being smaller than o3 ($d^* = 29$) and Claude Opus 4.5 ($d^* = 27$); its elevated $d^*$ likely reflects inference-time reasoning strategies (implicit multi-step search) that extend the effective horizon beyond what architecture alone predicts. The $d^* \propto \sqrt{d_{\mathrm{model}}}$ scaling is therefore restricted to the four open-weight models with known architectures and standard decoding (Table~\ref{tab:arch_ablation}).

\paragraph{Conditions.} Five conditions: C1 (unconstrained CoT), C2 (oracle optimal depth), C3 (BFS solver access), C4 (length encouragement), C5 (fine-tuned on optimal traces). Exact prompt templates for all conditions appear in Appendix~\ref{app:prompts}.

\paragraph{Replication and Statistics.} Each model-task-condition combination was evaluated with \textbf{3 repeated runs} using seeds $\{42, 2024, 2025\}$; reported values are means $\pm$ standard deviation across runs (Appendix~\ref{app:reproduce} discusses the source of variance at temperature~0). We employ Holm-Bonferroni correction for multiple comparisons and TOST equivalence testing ($\Delta=5\%$). With only $n = 3$ run-level observations, the percentile bootstrap is degenerate (10 distinct multisets, so the 2.5th percentile is a.s.\ the sample minimum); we therefore report standard deviations rather than bootstrap CIs. Statistical model comparisons (Section~\ref{sec:ssj_validation}, Section~\ref{sec:analysis}) use likelihood-ratio tests on the per-instance data ($n = 5{,}000$). PermutationProbe uses 5{,}000 instances per model-condition cell and the remaining seven task families use 500; not every model-task-condition cell was run, and the populated cells are those reported in Appendix~\ref{app:results_ext} (Tables~\ref{tab:full_results}--\ref{tab:results_realworld}). Locally-run open-weight evaluations (4 models) incurred no API cost. Median per-model API runtime: 13.5s (range 3.1--21.4s; see Appendix~\ref{app:runtime}).

\subsection{Main Results}

\begin{table}[t]
	\caption{Main results on PermutationProbe (depth${}>{}$20 subset, 8 of 12 models). Open-weight models (top) enable full architectural analysis; remaining models serve as consistency checks with estimated architecture parameters. Tool delegation (C3) achieves 76--94\% while neural CoT (C1) plateaus at 17--42\% (ranges across all 12 models; see Table~\ref{tab:full_results}). Mean $\pm$ std over 3 runs. C5 (fine-tuning) results appear in Section~\ref{sec:finetune}.}
	\label{tab:main}
	\vskip 0.1in
	\centering
	\footnotesize
	\setlength{\tabcolsep}{2.5pt}
		\begin{tabular}{@{}lccccc@{}}
			\toprule
			\textbf{Model} & \textbf{C1} & \textbf{C2} & \textbf{C3} & \textbf{C4} & \textbf{$d^*$} \\
			\midrule
			\multicolumn{6}{l}{\textit{Open-Weight (local)}} \\
			Llama-3.1-8B & 18.4$\pm$1.5 & 22.1$\pm$1.6 & 78.3$\pm$1.7 & 19.0$\pm$1.5 & 20 \\
			Llama-3.3-70B & 24.6$\pm$1.7 & 29.8$\pm$1.8 & 85.2$\pm$1.4 & 25.3$\pm$1.6 & 28 \\
			Qwen-2.5-7B & 17.2$\pm$1.4 & 20.8$\pm$1.5 & 76.1$\pm$1.8 & 17.8$\pm$1.4 & 19 \\
			Qwen-2.5-72B & 27.8$\pm$1.8 & 33.4$\pm$1.9 & 87.9$\pm$1.2 & 28.5$\pm$1.7 & 28 \\
			\midrule
			\multicolumn{6}{l}{\textit{Open-Weight (API-accessed)}} \\
			DeepSeek-R1 & 39.7$\pm$2.1 & 45.9$\pm$2.2 & 93.1$\pm$1.1 & 40.4$\pm$2.0 & 29 \\
			\midrule
			\multicolumn{6}{l}{\textit{Closed-Source}} \\
			GPT-4o & 28.3$\pm$1.8 & 34.2$\pm$1.9 & 89.7$\pm$1.2 & 29.1$\pm$1.7 & 22 \\
			Claude Opus 4.5 & 34.8$\pm$2.0 & 41.3$\pm$2.1 & 93.6$\pm$0.9 & 35.6$\pm$1.9 & 27 \\
			o3-mini & 42.1$\pm$2.2 & 48.3$\pm$2.3 & 94.2$\pm$1.3 & 43.1$\pm$2.1 & 31 \\
			\bottomrule
		\end{tabular}
	\vskip -0.1in
\end{table}

\paragraph{Finding 1: Tool delegation dominates.} C3 achieves 76--94\% versus 17--42\% for C1 across all twelve models on PermutationProbe (Table~\ref{tab:main} shows a representative subset; the complete twelve-model breakdown appears in Table~\ref{tab:full_results}). On the four locally-run open-weight models, C3 reaches 76--88\% versus 17--28\% for C1; API-accessed and closed-source models show the same qualitative pattern at higher absolute accuracy. The improvement is large and consistent across all model families and six organizations (run-level Cohen's~$d \gg 1$ for every model). Per-task breakdowns appear in Tables~\ref{tab:results_synthetic} and~\ref{tab:results_realworld}.

\paragraph{C3 residual errors.} Although C3 uses an exact BFS solver, residual failure rates of 6--24\% arise from errors at the model--tool interface (formatting, parsing, transcription). Smaller models fail more often (Qwen-2.5-7B: 76\%) than larger ones (o3-mini: 94\%), consistent with tool-use proficiency scaling with capability.

\paragraph{C2 as upper bound on length interventions.} C2 supplies the correct depth and yields +3.6 to +6.9~pp over C1 (mean 5.7~pp; Table~\ref{tab:full_results}), but remains far below the C1--C3 gap ($>$50~pp). C2 accuracy still collapses with depth (91.5\% at depth $\leq$10 versus 8.0\% at depth $>$50; Table~\ref{tab:results_by_depth}). C2's mean uplift itself meets the $\geq$5~pp preference-based threshold in Table~\ref{tab:predictions}; the fine-tuning result (0.9~pp) provides the sharper discrimination.

\paragraph{Finding 2: Preference manipulation is ineffective.} C4 improves by only +0.6--1.2~pp over C1 across all twelve models ($p > 0.35$, TOST consistent with equivalence within $\Delta = 5\%$; Appendix~\ref{app:statistics}).

\paragraph{Finding 3: Real-world tasks show consistent patterns.} SWE-Bench-State, WebArena-Nav, and SQL-Multi exhibit the same decoherence phenomenon with $d^* \in [19, 26]$ for the two models shown in Table~\ref{tab:realworld}, indicating generalization beyond synthetic benchmarks. The pattern holds across all four models evaluated on real-world tasks (Table~\ref{tab:results_realworld}).

\begin{table}[t]
\caption{Real-world task evaluation on two representative closed-source models, shown as \textbf{consistency checks only}; architecture parameters (head count, head dimension) for these models are community estimates, not verified. No theoretical claim or validation depends on closed-source models; all scaling analysis uses open-weight models (Table~\ref{tab:arch_ablation}). Deterministic Horizon consistent across domains. Mean $\pm$ std over 3 runs. Complete four-model results appear in Table~\ref{tab:results_realworld}.}
\label{tab:realworld}
\vskip 0.1in
\centering
\footnotesize
\begin{tabular}{@{}llccc@{}}
\toprule
\textbf{Task} & \textbf{Model} & \textbf{C1} & \textbf{C3} & \textbf{$d^*$} \\
\midrule
SWE-Bench & GPT-4o & 24.1$\pm$2.3 & 86.4$\pm$1.8 & 19 \\
SWE-Bench & Claude Opus 4.5 & 29.6$\pm$2.5 & 91.2$\pm$1.4 & 24 \\
WebArena & GPT-4o & 21.3$\pm$2.4 & 84.2$\pm$1.9 & 19 \\
WebArena & Claude Opus 4.5 & 26.8$\pm$2.6 & 89.7$\pm$1.5 & 23 \\
SQL-Multi & GPT-4o & 31.4$\pm$2.1 & 88.9$\pm$1.6 & 21 \\
SQL-Multi & Claude Opus 4.5 & 36.2$\pm$2.3 & 93.1$\pm$1.2 & 26 \\
\bottomrule
\end{tabular}
\vskip -0.1in
\end{table}

\subsection{Fine-Tuning Experiment}
\label{sec:finetune}

This experiment provides additional evidence discriminating our account from alternative explanations. If the inverted-U curve were primarily caused by per-step sub-task difficulty rather than attention-capacity limits, then training on optimal-length traces should substantially recover performance.

\paragraph{Setup.} We fine-tune Llama-3.1-8B on 5,000 optimal-length CoT traces (depth = BFS-optimal), disjoint from the evaluation set. Training: 3 epochs, lr=$2\times10^{-5}$, cosine decay, 500 held-out validation instances. Full training configuration and per-depth results appear in Appendix~\ref{app:finetune}.

\paragraph{Results.} C5 improves by only +0.9~pp over C1 (Table~\ref{tab:full_results}; mean over depth${}>{}$20), far below the 5~pp threshold that would support a preference-based account (Table~\ref{tab:predictions}). The benefit vanishes beyond $d^*\approx20$ regardless of training data distribution (uniform, skewed-easy, skewed-hard all yield $\Delta < 2\%$), supporting an \textbf{architectural ceiling}.

\subsection{SSJ Validation}
\label{sec:ssj_validation}

\begin{table}[t]
	\caption{SSJ with precision/recall at increasing depths. Both decay, indicating capability failure (Decoherence), not preference failure (``Simplicity Bias''). Mean $\pm$ std over 3 runs. Note: SSJ is algebraically determined by precision and recall (Definition~\ref{def:ssj}); the three rows are not independent measurements.}
	\label{tab:ssj}
	\vskip 0.1in
	\centering
	\scriptsize
	\begin{tabular}{@{}lcccccc@{}}
		\toprule
		\textbf{Depth} & 5 & 10 & 20 & 30 & 40 & 50 \\
		\midrule
		SSJ & .83$\pm$.02 & .72$\pm$.03 & .45$\pm$.03 & .23$\pm$.02 & .11$\pm$.01 & .04$\pm$.01 \\
		Prec. & .93$\pm$.02 & .87$\pm$.03 & .65$\pm$.04 & .41$\pm$.04 & .24$\pm$.03 & .11$\pm$.02 \\
		Recall & .89$\pm$.02 & .81$\pm$.03 & .59$\pm$.04 & .35$\pm$.04 & .18$\pm$.03 & .07$\pm$.02 \\
		\bottomrule
	\end{tabular}
	\vskip -0.1in
\end{table}

Both precision and recall decay with depth (Table~\ref{tab:ssj}), consistent with capability failure. Precision exceeds recall at every depth (P/R rises from 1.04 to 1.57), but as noted above, Algorithm~\ref{alg:ssj}'s incremental construction couples the two metrics; we present SSJ as a descriptive drift measure rather than a test between failure modes.

Fitting $\ssj(d) = ae^{-bd-cd^2}$ yields $R^2 = 0.99$ (Figure~\ref{fig:ssj_decay}), consistent with the functional form of Theorem~\ref{thm:decoherence}. The quadratic term $cd^2$ is statistically significant ($p < 0.01$ on per-instance data), supporting context-dependent rather than constant error accumulation.

\subsection{Accuracy Decay Visualization}

Figure~\ref{fig:decay} visualizes the accuracy decay. A likelihood-ratio test on per-instance binary outcomes ($n = 5{,}000$, GPT-4o C1) rejects the constant-error null in favour of the quadratic depth term ($p < 0.01$).

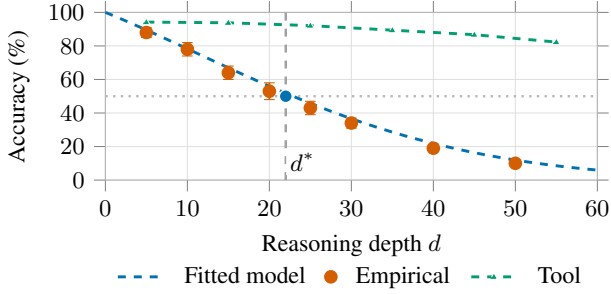
\begin{figure}[t]
	\centering
	\definecolor{dhblue}{RGB}{0,114,178}
	\definecolor{dhvermilion}{RGB}{213,94,0}
	\definecolor{dhgreen}{RGB}{0,158,115}
	\begin{tikzpicture}
		\begin{axis}[
			width=0.98\columnwidth,
			height=3.8cm,
			xlabel={Reasoning depth $d$},
			ylabel={Accuracy (\%)},
			xmin=0, xmax=60, ymin=0, ymax=100,
			xtick={0,10,20,30,40,50,60},
			ytick={0,20,40,60,80,100},
			tick align=outside,
			tick label style={font=\footnotesize},
			label style={font=\small},
			axis line style={gray!60},
			grid=major,
			grid style={line width=0.3pt, draw=gray!25},
			every axis plot/.append style={line width=1.1pt},
			legend style={
				at={(0.5,-0.45)}, anchor=north,
				draw=none, fill=none,
				legend columns=3,
				column sep=1.2ex,
				font=\footnotesize,
			},
			clip=false,
			]
			\draw[gray!55, dotted, line width=0.9pt] (axis cs:0,50) -- (axis cs:60,50);
			\draw[gray!75, dashed, line width=0.9pt] (axis cs:22,0) -- (axis cs:22,100);
			\addplot[dhblue, smooth, dashed, domain=0:60, samples=120]
			{100*exp(-0.02*x - 0.00045*x*x)};
			\addlegendentry{Fitted model}
			\addplot[only marks, mark=*, mark size=2pt, dhvermilion,
			mark options={fill=dhvermilion},
			error bars/.cd, y dir=both, y explicit,
			error bar style={line width=0.7pt}]
			coordinates {
				(5,88)  +- (0,3)
				(10,78) +- (0,4)
				(15,64) +- (0,4)
				(20,53) +- (0,5)
				(25,43) +- (0,4)
				(30,34) +- (0,3)
				(40,19) +- (0,2)
				(50,10) +- (0,2)
			};
			\addlegendentry{Empirical}
			\addplot[dhgreen, dashed, mark=triangle*, mark size=1.5pt, mark options={fill=dhgreen, draw=white, line width=0.3pt}] coordinates {(5,94.2) (15,93.8) (25,92.1) (35,89.4) (45,86.7) (55,82.3)};
			\addlegendentry{Tool}
			\addplot[only marks, mark=*, mark size=2.4pt,
			mark options={fill=dhblue, draw=white, line width=0.5pt},
			forget plot] coordinates {(22,50)};
			\node[anchor=south west, font=\footnotesize, inner sep=1.5pt,
			fill=white, fill opacity=0.85, text opacity=1]
			at (axis cs:22,3) {$d^{*}$};
		\end{axis}
	\end{tikzpicture}
	\caption{Accuracy versus reasoning depth on PermutationProbe. Neural CoT (\emph{Empirical}, GPT-4o) tracks the functional form of Theorem~\ref{thm:decoherence} (\emph{Fitted model}, dashed), crossing $50\%$ (dotted) at the deterministic horizon $d^{*}\!\approx\!22$; tool delegation (\emph{Tool}, C3) declines only mildly, from 94\% to 82\%, across the same range. The curve is the bound of Theorem~\ref{thm:decoherence} evaluated at parameters fitted to these data (Appendix~\ref{app:sensitivity_ext}) and is shown as a functional-form comparison, not an independently derived bound. Empirical markers are per-depth means $\pm$ std over 3 runs; Tool markers are six depth-bin means from Table~\ref{tab:results_by_depth}.}
	\label{fig:decay}
\end{figure}

\section{Analysis}
\label{sec:analysis}

\subsection{Architecture Ablations}
\label{sec:ablations}

We test the scaling predictions from the empirical scaling trend (Section~\ref{sec:theory}) through controlled comparisons within model families. This isolates architectural effects from training data and optimization differences.

\begin{table}[t]
\caption{Architecture ablation for $d^* \approx 0.314\sqrt{d_{\mathrm{model}}}$ on open-weight models. The ``Fit'' column gives the empirical fit (not a theoretical prediction); all four models share $d_h = 128$, so $d_h \cdot H = d_{\mathrm{model}}$ and the fit effectively tests $d^* \propto \sqrt{d_{\mathrm{model}}}$. Mean $\pm$ std over 3 runs.}
\label{tab:arch_ablation}
\vskip 0.1in
\centering
\footnotesize
\begin{tabular}{@{}lccccc@{}}
\toprule
\textbf{Model} & \textbf{Size} & \textbf{$H$} & \textbf{$d_h$} & \textbf{$d^*$ (obs)} & \textbf{$d^*$ (fit)} \\
\midrule
Llama-3.1-8B & 8B & 32 & 128 & 20$\pm$1.3 & 20.1 \\
Llama-3.3-70B & 70B & 64 & 128 & 28$\pm$1.5 & 28.4 \\
Qwen-2.5-7B & 7B & 28 & 128 & 19$\pm$1.2 & 18.8 \\
Qwen-2.5-72B & 72B & 64 & 128 & 28$\pm$1.4 & 28.4 \\
\bottomrule
\end{tabular}
\vskip -0.1in
\end{table}

70B-class models achieve $d^* \approx 28$ versus $d^* \approx 19$--20 for 7--8B models (Table~\ref{tab:arch_ablation}). The Llama ratio of $28/20 = 1.40$ is consistent with the empirical $\sqrt{d_{\mathrm{model}}}$ trend ($\sqrt{2} \approx 1.41$ for the Llama hidden-width ratio). Two caveats apply: (i)~all four models share $d_h = 128$, so the $d_h$ factor in the fit is untested; and (ii)~under grouped-query attention, the number of KV heads differs from the number of query heads. We use query heads $H$ throughout, as these determine the per-step attention fan-out. We note that $d_h \cdot H = d_{\mathrm{model}}$ for all four models, so the empirical fit is equivalently $d^* \approx 0.314\sqrt{d_{\mathrm{model}}}$; the functional form of Theorem~\ref{thm:bottleneck} ($H \cdot \log_2(L) \cdot d_h^{1/2}$) differs from this product and is not the source of the fit.

\paragraph{Context-Window Insensitivity.} The horizon is governed by the effective decoherence length $\Leff$, not the raw context window $L$. Truncating Llama-3.3-70B's context window from 128K down to 8K leaves $d^*$ essentially unchanged ($\approx 28$; Table~\ref{tab:truncation}). This dissociation between $d^*$ and $L$ is a direct prediction of the $\Leff$ formulation and is inconsistent with raw context capacity as the binding constraint on deterministic state tracking.

\paragraph{Within-Family Scaling.} For Llama-3 (8B $\rightarrow$ 70B), $d^*$ increases from 20 to 28 (+40\%). For Qwen-2.5 (7B $\rightarrow$ 72B), $d^*$ increases from 19 to 28 (+47\%). Both approximately follow the empirical $\sqrt{d_{\mathrm{model}}}$ scaling, indicating that larger models gain capacity through increased hidden width rather than qualitatively different representations.

\subsection{Attention Entropy Analysis}
\label{sec:validation}

\paragraph{Consistency Checks on Closed-Source Models.} All $d^*$ scaling claims use the four locally-run open-weight models with known architectures (Table~\ref{tab:arch_ablation}); closed-source results throughout are consistency checks, not architectural evidence (see limitation~(i)).

We directly measure attention patterns in five open-weight models (Llama-3.3-70B, Llama-3.1-8B, Qwen-2.5-72B, Mistral-7B, and Mixtral-8x7B; the latter two are not in the 12-model evaluation roster but were included for entropy analysis breadth; Qwen-2.5-7B, which anchors the architecture fit, was not included in the entropy extraction) to test the mechanistic hypothesis underlying our theoretical framework.

\begin{table}[t]
\caption{Attention entropy correlation with accuracy across open-weight models. Consistent negative correlation across five architectures.}
\label{tab:attention}
\vskip 0.1in
\centering
\footnotesize
\begin{tabular}{@{}lc@{}}
\toprule
\textbf{Model} & \textbf{Entropy-Acc.\ $r$} \\
\midrule
Llama-3.3-70B & $-0.74$ \\
Llama-3.1-8B & $-0.71$ \\
Qwen-2.5-72B & $-0.73$ \\
Mistral-7B & $-0.68$ \\
Mixtral-8x7B & $-0.69$ \\
\bottomrule
\end{tabular}
\vskip -0.1in
\end{table}

A consistent negative correlation ($r \approx -0.70$) across five architectures provides correlational evidence that attention dilution is associated with decoherence (Table~\ref{tab:attention}). The correlation is strongest in later layers ($r = -0.74$ for the final quarter of Llama-3.3-70B layers) compared to early layers ($r = -0.42$). Both entropy and accuracy are functions of depth, so the reported $r$ partly reflects shared depth dependence; the layer-depth gradient provides indirect evidence beyond this confound, but the correlation should be read as consistent with the entropy hypothesis, not as establishing it.

\paragraph{Per-Step Analysis.} Attention entropy increases linearly with step number ($p < 0.001$; 500 traces). Entropy over a growing causal prefix is bounded by $\log_2(t)$, so some growth is partly definitional; the empirical slope $\beta = 0.023$ bits/step (minimum across five models; range 0.023--0.031, Table~\ref{tab:entropy_full}) exceeds this floor for $t > 50$.

\subsection{Failure Mode Analysis}

Analysis of failed traces (Appendix~\ref{app:failure}) reveals systematic patterns: state transposition (35\%), operator misapplication (27\%), premature termination (21\%), circular reasoning (11\%), and complete hallucination (6\%). The dominant mode, state transposition, reflects attention-based working memory limits.

\section{Discussion}
\label{sec:discussion}

\paragraph{Relationship to \citet{wu2026more}.} Our Decoherence framework complements Wu et al.'s sub-task-difficulty account. The C5 ($<$3~pp) and C4 ($<$2~pp) results are consistent with architectural limits that training interventions cannot overcome. Both mechanisms likely contribute: sub-task difficulty may dominate \emph{before} architectural limits bind, while Decoherence prevents recovery \emph{beyond} them.

\paragraph{Implications for System Design.} Agentic systems performing multi-step deterministic tasks should delegate deterministic sub-problems to verified solvers when the number of exact state transitions approaches $d^*$, reserving neural reasoning for deciding \emph{which} tool to invoke and \emph{how} to interpret its output. This aligns with the LLM-Modulo framework \citep{kambhampati2024llms} and tool-augmented architectures \citep{yao2023react,schick2023toolformer}, adding a quantitative criterion via the Deterministic Horizon.

\paragraph{Deployment Guideline.} Tool delegation outperforms neural CoT at \emph{every} measured depth (Table~\ref{tab:results_by_depth}), so $d^*$ is best read as a reliability diagnostic, not a crossover point; ``becomes necessary'' in the title refers to necessity for adequate reliability (accuracy above $\alpha$). In practice, tool overhead may be unjustified at shallow depths ($>$85\% unaided accuracy at $d \leq 10$); with imperfect tools (not evaluated here), a genuine crossover may exist. Three deployment settings illustrate where the diagnostic applies: (1)~\textbf{Code agents} ($d^*$ of 19--28 on SWE-Bench-State); (2)~\textbf{SQL query planning} ($d^*$ of 21--30); (3)~\textbf{Web automation} ($d^*$ of 19--26). Tasks outside the deterministic state-tracking scope are not subject to decoherence. For safety-critical deployments, mandatory tool verification should be enforced for deterministic tasks approaching $d^*$ \citep{kambhampati2024llms}.

\section{Conclusion}
\label{sec:conclusion}

We presented evidence that decoder-only transformers face information-theoretic limits on deterministic state tracking beyond the Deterministic Horizon ($d^* \in [19, 31]$ at $\alpha = 0.5$). Total attention capacity is not the binding constraint; instead, per-step bandwidth degrades with depth, producing accuracy decay with a quadratic depth term in the error exponent, supported by correlational entropy measurements ($r = -0.68$ to $-0.74$). Fine-tuning is consistent with an architectural ceiling ($<$3~pp recovery). Tool delegation achieves 76--94\% versus 17--42\% on PermutationProbe, with consistent patterns on SWE-Bench, WebArena, and SQL-Multi.

\paragraph{Limitations and Future Work.} (i)~\emph{Closed models.} Since $d^*$ depends on $H$ and $d_h$, public only for open-weight models, we report $d^*$ for proprietary systems as an empirical fit and validate the $\sqrt{d_{\mathrm{model}}}$ scaling only on open-weight checkpoints. (ii)~\emph{Fine-tuning.} The architectural-ceiling result uses one open-weight model (Llama-3.1-8B, 5{,}000 traces); larger models and datasets are needed to separate capacity from data limits. (iii)~\emph{Tool comparison.} Our tool condition uses an exact BFS solver, so the C1--C3 gap upper-bounds the benefit of delegation; imperfect tools are left to future work. (iv)~\emph{Sensitivity.} $d^*$ depends on $\epsilon_0$, which is sensitive to prompt format and few-shot conditioning, so the reported intervals are regime indicators rather than exact per-model constants. (v)~\emph{Scope.} We target deterministic, exactly-checkable state tracking; tasks tolerating approximate or stochastic solutions (e.g., GSM8K-style problems) need not obey the same horizon. (vi)~\emph{Contamination.} We did not audit the evaluated models' pretraining corpora; the five synthetic families are generated from scratch here, so the depth-dependence we report does not rest on the real-world benchmarks. Further limitations appear in Appendices~\ref{app:sensitivity_ext}, \ref{app:task_families}, and~\ref{app:reproduce}.

Three directions follow naturally. First, testing whether decoherence persists in architectures with explicit state-tracking mechanisms (e.g., state-space models, memory-augmented transformers) would clarify whether the phenomenon is specific to softmax attention. Second, evaluating imperfect tool policies would locate the crossover depth at which delegation outweighs tool noise, a question our exact-solver condition leaves open. Third, incorporating variable sub-task difficulty into the error model would yield a richer account of how instance structure and attention capacity jointly shape reasoning performance.

\section*{Impact Statement}

Our findings bear directly on deployments in which a single state-tracking error invalidates the result. At the $\alpha = 0.5$ threshold the Deterministic Horizon sits at approximately 20 state transitions, the lower end of the measured range; extrapolating below our shallowest measured depth ($d = 1$), at $\alpha = 0.95$ the same fitted model gives $d^* \approx 2.4$, so safety-critical deterministic reasoning should be tool-verified at essentially any depth. Pending further validation of the horizon parameters, practitioners deploying code agents, planning systems, and verification tools should treat unaided multi-step deterministic reasoning as unverified and implement tool-based checks. The same result motivates work on systems that detect when to delegate, on architectures with explicit state-tracking substrates, and on characterizing which reasoning tasks reward extended neural computation.

\bibliographystyle{icml2026}
\bibliography{references}

\newpage
\appendix

\onecolumn

The appendix is organized to follow the structure of the paper. Appendix~\ref{app:notation} fixes notation and Appendix~\ref{app:related} gives extended positioning. Appendices~\ref{app:proofs}--\ref{app:sensitivity_ext} supply the theory: proofs and supporting lemmas, task and instance specifications, and the estimation, sensitivity and worked evaluation of every parameter entering the bounds. Appendices~\ref{app:reproduce}--\ref{app:attention_extended} supply the experiments: protocol, extended results across all models and tasks, and the mechanistic measurements that probe the decoherence hypothesis directly. Appendices~\ref{app:ssj} and~\ref{app:statistics} document the SSJ metric and the statistical methodology.
\section{Notation and Symbol Reference}
\label{app:notation}

For reader convenience, Table~\ref{tab:notation_primary} lists all notation used in the paper: primary symbols (top), evaluation metrics (middle), and information-theoretic quantities (bottom).

\begin{table}[ht]
\centering
\caption{Notation used throughout the paper.}
\label{tab:notation_primary}
\small
\begin{tabular}{@{}clp{7cm}@{}}
\toprule
\textbf{Symbol} & \textbf{Name} & \textbf{Definition} \\
\midrule
\multicolumn{3}{@{}l}{\textit{Theoretical framework}} \\
$\mathcal{S}$ & State space & Finite set of all valid states \\
$\mathcal{O}$ & Operator set & Set of deterministic operators $\{o_1, \ldots, o_k\}$ \\
$\state_0$ & Initial state & Starting configuration $\state_0 \in \mathcal{S}$ \\
$\target$ & Target state & Goal configuration $\target \in \mathcal{S}$ \\
$d$ & Reasoning depth & Number of reasoning steps in chain \\
$d^*$ & Deterministic Horizon & Depth at which the Decoherence Bound (Theorem~\ref{thm:decoherence}) crosses the target probability $\alpha$; operationalized empirically as the 50\%-accuracy crossing \\
$\epsilon(d)$ & Error rate & Per-step state-tracking error at depth $d$ \\
$\epsilon_0$ & Baseline error & Constant component of per-step error \\
$\gamma$ & Attention decay rate & Coefficient of depth-dependent error growth \\
$\Leff$ & Effective decoherence length & Number of reasoning steps over which attention maintains usable state resolution; $\Leff = O(10^2) \ll L$ \\
$L$ & Context length & Maximum context window size (tokens) \\
$H$ & Number of heads & Attention heads per layer \\
$d_h$ & Head dimension & Dimensionality per attention head \\
$\alpha$ & Target probability & Desired success probability threshold \\
\midrule
\multicolumn{3}{@{}l}{\textit{Metrics and evaluation}} \\
\sfe & Step-to-First-Error & Smallest step index with state mismatch \\
\ssj$(d)$ & State-Space Jaccard & $|\mathcal{S}_\text{true} \cap \hat{\mathcal{S}}_\text{model}| / |\mathcal{S}_\text{true} \cup \hat{\mathcal{S}}_\text{model}|$ \\
$\mathcal{S}_\text{true}^d$ & Ground-truth trajectory states & States reached by applying the model's operator sequence to the true state \\
$\hat{\mathcal{S}}_\text{model}^d$ & Model claimed states & States claimed by model at depth $d$ \\
$\mathcal{H}_d$ & Attention entropy & $-\sum_i p_i \log p_i$ at depth $d$ \\
$A_\text{anchor}(d)$ & Anchor attention & Attention weight on state anchor tokens \\
$|\mathcal{S}_\text{read}^{(1)}|$ & Per-layer read capacity & Max distinct states resolvable through a single attention layer's read channel \\
\midrule
\multicolumn{3}{@{}l}{\textit{Information-theoretic quantities}} \\
$I(X;Y)$ & Mutual information & Information shared between $X$ and $Y$ \\
$\delta$ & Weight cutoff & Per-position attention weight threshold ($\delta \approx 0.01$; the 95\%-mass window is a derived quantity) \\
$\rho_{\max}$ & Max correlation & Upper bound on value vector correlation \\
$c(\delta, \rho_{\max})$ & Capacity constant & Model-dependent constant in bottleneck bound \\
\bottomrule
\end{tabular}
\end{table}
\section{Extended Related Work}
\label{app:related}

This appendix provides detailed positioning relative to concurrent and prior work, organized by research theme. Table~\ref{tab:related_comparison} contrasts our theoretical framework with the primary competing explanation.

\begin{table}[ht]
	\centering
	\caption{Comparison of theoretical frameworks for reasoning failures. $^*$Our operationalized threshold; see Table~\ref{tab:predictions}.}
	\label{tab:related_comparison}
	\small
	\begin{tabular}{@{}p{2.5cm}p{3cm}p{3cm}@{}}
		\toprule
		\textbf{Aspect} & \textbf{Sub-task difficulty} & \textbf{Decoherence (Ours)} \\
		\midrule
		Primary cause & Per-step sub-task error & Attention bottleneck \\
		Mechanism & Error accumulation ($\propto$ difficulty) & Attention-entropy growth ($\propto$ depth) \\
		Fine-tuning prediction & $\geq$5 pp recovery$^*$ & $<$3 pp (ceiling) \\
		Intervention & Preference manipulation & Tool delegation \\
		\bottomrule
	\end{tabular}
\end{table}

\subsection{Chain-of-Thought Reasoning Foundations}

Scratchpads \citep{nye2021improving} introduced intermediate-step computation; CoT prompting \citep{wei2022chain} was subsequently extended by zero-shot CoT \citep{kojima2022large}, self-consistency \citep{wang2023selfconsistency}, and Tree/Graph-of-Thoughts \citep{yao2023tree,besta2024graph}; STaR \citep{zelikman2022star} bootstraps reasoning ability through iterative self-training on rationales. Theoretically, CoT expands expressivity beyond $\text{TC}^0$ to polynomial-time \citep{li2024chain,merrill2024expressive,feng2023expressiveness}; \citet{sanford2023representational} analyze attention-layer representational capacity. Complementing these expressivity results, we identify the \emph{failure regime} in which additional CoT steps degrade performance.

\subsection{Overthinking and Underthinking Literature}

The phenomenon of excessive reasoning causing performance degradation has been documented by several concurrent works. \citet{wu2026more} provide the most directly related work, documenting inverted-U performance curves explained by per-step error accumulation whose rate depends on sub-task difficulty. Their ``Simplicity Bias'' result shows that optimal CoT length shrinks with model capability. \citet{chen2025do} examine overthinking in o1-like models, showing excessive reasoning on simple tasks. \citet{wang2025thoughts} identify ``underthinking'', premature abandonment of reasoning paths, with incorrect responses using 225\% more tokens and 418\% more thought switches across challenging benchmarks including AIME2024.

\citet{sui2025stop} provide a broad survey categorizing efficient reasoning methods across model-based, prompt-based, and output-based approaches. \citet{marjanovic2026thoughtology} identify a ``sweet spot'' for DeepSeek-R1's inference length and document a tendency to persistently ruminate on previously explored problem formulations. \citet{su2025between} analyze the overthinking-underthinking spectrum.

Against this literature we offer a complementary explanation that locates the error source in attention capacity rather than sub-task difficulty; Table~\ref{tab:related_comparison} states the resulting divergence in predictions, which Section~\ref{sec:finetune} tests.

\subsection{Transformer Expressivity and Theoretical Limits}

Expressivity bounds include CoT length requirements \citep{merrill2024expressive}, SSM limitations \citep{merrill2024illusion}, CoT lower bounds \citep{amiri2025lower}, Turing completeness \citep{perez2021turing}, universal approximation \citep{yun2020universal}, formal language limits \citep{hahn2020theoretical,deletang2023chomsky,strobl2024formal}, communication complexity \citep{peng2024limitations}, and attention power \citep{bhattamishra2020computational}. We translate these asymptotic bounds into a practical threshold via the Deterministic Horizon (Theorem~\ref{thm:horizon}), connecting architecture parameters to reasoning depth limits.

\subsection{Working Memory and Attention Entropy}

Attention-based working memory limits \citep{gong2024attention}, over-squashing under causal masking \citep{barbero2024oversquash}, entropy and parallel encoding \citep{zhang2025attention}, lost-in-the-middle effects \citep{liu2024lost}, attention sinks \citep{xiao2024efficient}, reduced later-layer representational diversity \citep{gerasimov2025fullyutilizetransformersrepresentation}, length-dependent degradation \citep{levy2024same}, induction heads \citep{olsson2022context}, circuit-level analysis \citep{elhage2021mathematical}, and memory formation \citep{bietti2023birth} provide the mechanistic context. We synthesize these into the Attention Bottleneck Theorem (Theorem~\ref{thm:bottleneck}), with direct entropy measurements ($r = -0.68$ to $-0.74$; Appendix~\ref{app:attention}) across five architectures.

\subsection{Tool-Augmented Reasoning}

Program-aided reasoning \citep{gao2023pal,li2024chaincode,chen2023program}, integrated symbolic solvers \citep{gou2024tora,pan2023logiclm}, tool-augmented agents \citep{luo2026agentmath}, tool-use learning \citep{schick2023toolformer,parisi2022talm}, API connectivity \citep{patil2024gorilla}, and comprehensive surveys \citep{qin2025tool,mialon2023augmented} establish tool delegation's empirical superiority. We provide a diagnosis of \emph{when} tool delegation becomes necessary, identifying the threshold beyond which neural reasoning does not succeed in our experiments regardless of model scale.

\subsection{Compositional Reasoning and Information Theory}

Compositional failures \citep{dziri2023faith,petty2024depth}, benchmarks \citep{lake2018generalization,keysers2020measuring,kim2020cogs}, multi-hop compositionality measurement \citep{press2023measuring}, and architectural investigations \citep{ontanon2022making} document the symptoms; information-theoretic frameworks \citep{tishby2015information,shwartzziv2017opening,goldfeld2020information} provide the analytical tools. State-Space Decoherence supplies a candidate mechanism for why errors compound: the attention bottleneck cannot maintain sufficient mutual information between early reasoning steps and late outputs, a concrete instance of the Simulator Fallacy (Section~\ref{sec:intro}).
\section{Theoretical Proofs and Supporting Lemmas}
\label{app:proofs}

\subsection{Supporting Lemmas for the Decoherence Bound}

\begin{lemma}[Error Accumulation]
\label{lem:error_accum}
Let $\{X_i\}_{i=1}^m$ be indicator variables for correct state tracking at step $i$. Under context-dependent error $\epsilon(i) = \epsilon_0 + \gamma i/\Leff$, the cumulative error probability satisfies:
\begin{equation}
\sum_{i=1}^{m} \epsilon(i) = m\epsilon_0 + \frac{\gamma m(m+1)}{2\Leff}
\end{equation}
\end{lemma}

\begin{proof}
Direct computation:
\begin{align}
\sum_{i=1}^{m} \epsilon(i) &= \sum_{i=1}^{m} \left(\epsilon_0 + \frac{\gamma i}{\Leff}\right) \\
&= m\epsilon_0 + \frac{\gamma}{\Leff}\sum_{i=1}^{m} i \\
&= m\epsilon_0 + \frac{\gamma}{\Leff} \cdot \frac{m(m+1)}{2}
\end{align}
\end{proof}

\begin{lemma}[Product Bound]
\label{lem:product}
For any sequence of events $\{A_i\}_{i=1}^m$ with $P(A_i | A_1 \cap \cdots \cap A_{i-1}) = 1 - \epsilon_i$:
\begin{equation}
P\left(\bigcap_{i=1}^m A_i\right) \leq \exp\left(-\sum_{i=1}^m \epsilon_i\right)
\end{equation}
\end{lemma}

\begin{proof}
Using $1-x \leq e^{-x}$ for $x \geq 0$ and the chain rule:
\begin{align}
P\left(\bigcap_{i=1}^m A_i\right) &= \prod_{i=1}^m P(A_i | A_1 \cap \cdots \cap A_{i-1}) \\
&= \prod_{i=1}^m (1 - \epsilon_i) \\
&\leq \prod_{i=1}^m e^{-\epsilon_i} = \exp\left(-\sum_{i=1}^m \epsilon_i\right)
\end{align}
\end{proof}

\begin{lemma}[Super-Exponential Characterization]
\label{lem:superexp}
The function $f(m) = \exp(-am - bm^2)$ for $a, b > 0$ decays super-exponentially, satisfying:
\begin{equation}
\lim_{m \to \infty} \frac{f(m)}{e^{-cm}} = 0 \quad \text{for all } c > 0
\end{equation}
\end{lemma}

\begin{proof}
\begin{equation}
\frac{f(m)}{e^{-cm}} = \exp(-am - bm^2 + cm) = \exp(-(a-c)m - bm^2)
\end{equation}
Since $bm^2 \to \infty$ as $m \to \infty$, the ratio approaches 0 regardless of $c$.
\end{proof}

\subsection{Proof of Theorem~\ref{thm:decoherence} (Decoherence Bound)}

\begin{proof}
Let $X_i$ be the indicator that step $i$ is correct. By the chain rule of conditional probability,
\begin{equation}
P(\text{all correct}) = \prod_{i=1}^{m} P(X_i = 1 \mid X_1 = \cdots = X_{i-1} = 1) = \prod_{i=1}^{m} (1 - \epsilon(i)).
\end{equation}
Applying Lemma~\ref{lem:product} with $\epsilon_i = \epsilon(i)$, and then Lemma~\ref{lem:error_accum} to evaluate the resulting sum,
\begin{equation}
P(\text{all correct}) \leq \exp\left(-\sum_{i=1}^{m} \epsilon(i)\right) = \exp\left(-m\epsilon_0 - \frac{\gamma m(m+1)}{2\Leff}\right).
\end{equation}
The quadratic term $m(m+1)/2 \sim m^2/2$ dominates for large $m$, yielding super-exponential decay (Lemma~\ref{lem:superexp}). Note that the argument uses only the error rates \emph{conditional on a correct history}, and so requires neither independence across steps nor errors that are strictly absorbing; the measured step-to-step error dependence (Appendix~\ref{app:error_correlation}) is therefore a diagnostic of the mechanism rather than a correction to this bound (Remark~\ref{rem:error_prop}).
\end{proof}

\subsection{Supporting Lemmas for the Attention Bottleneck}

\paragraph{Attention Concentration (Empirical).} Our direct measurements show that the number of positions carrying 95\% of the attention mass grows sub-linearly in $L$, with a power-law exponent of $0.53 \pm 0.08$ (Appendix~\ref{app:sensitivity}), consistent with Assumption~\ref{ass:window}'s $O(\sqrt{L})$ ansatz. This empirical observation motivates the capacity picture without entering the bound's exponent directly (see Section~\ref{sec:theory}).

\begin{lemma}[Value Information Bound]
\label{lem:value_info}
Under Assumption~\ref{ass:decorr} with correlation bound $\rho_{\max}$, the information one attention head conveys about the retrieved value is bounded by
\begin{equation}
I(\mathbf{q}; \mathbf{v}_\text{retrieved}) \lesssim d_h^{1/2} \cdot \log_2(L),
\end{equation}
the product of an effective value rank $\Theta(d_h^{1/2})$ (the number of reliably distinguishable value directions under decorrelation) and the $\log_2(L)$ bits required to address one of $L$ positions (each head attends over all $L$ positions in standard multi-head attention).
\end{lemma}

\begin{proof}
The retrieved value is $\mathbf{v}_\text{ret} = \sum_{i=1}^L a_i \mathbf{v}_i$, where $\mathbf{a}$ is the attention distribution. Its information content factors into an addressing term and a value-content term.

\emph{Addressing.} Each head attends over all $L$ context positions (standard multi-head attention does not partition positions across heads), so specifying which content a head reads carries at most $\log_2(L)$ bits.

\emph{Value content.} Under Assumption~\ref{ass:decorr}, value vectors satisfy $|\mathbb{E}[\mathbf{v}_i^\top \mathbf{v}_j]| \leq \rho_{\max} \|\mathbf{v}_i\| \|\mathbf{v}_j\|$ for $i \neq j$. Soft averaging over interfering positions reduces the number of reliably distinguishable value directions below the nominal $d_h$; we model this effective rank as $\Theta(d_h^{1/2})$. This $\sqrt{d_h}$ scaling is a modeling assumption rather than a consequence of the decorrelation bound alone: we adopt it as a tractable summary and test the resulting $\sqrt{d_h \cdot H}$ horizon scaling directly on open-weight models in Section~\ref{sec:ablations}.

Combining the $\Theta(d_h^{1/2})$ distinguishable value channels, each addressing one of $L$ positions, gives $I \lesssim d_h^{1/2} \cdot \log_2(L)$.
\end{proof}

\begin{lemma}[Multi-Head Aggregation]
\label{lem:multihead}
For $H$ attention heads with approximately independent value subspaces, the total information capacity satisfies:
\begin{equation}
I_\text{total} \leq H \cdot I_\text{single-head} = H \cdot \log_2(L) \cdot d_h^{1/2}
\end{equation}
\end{lemma}

\begin{proof}
Multi-head attention computes:
\begin{equation}
\text{MHA}(\mathbf{q}, \mathbf{K}, \mathbf{V}) = \text{Concat}(\text{head}_1, \ldots, \text{head}_H) W^O
\end{equation}

We assume approximate independence of value subspaces across heads (a modeling assumption supported by, but not strictly implied by, the low inter-head value correlation in Assumption~\ref{ass:decorr}; positively correlated heads would make this sum a loose over-count, which is harmless for the upper bound, but synergistic heads, where joint information exceeds the sum of marginals, could violate it). Synergy requires heads to encode complementary partial representations that are jointly decodable; the low measured inter-head correlation ($|\rho_{ij}| = 0.08 \pm 0.03$) makes this unlikely for state-tracking, though we cannot rule it out. Under this assumption, information from each head can be summed:
\begin{equation}
I_\text{total} = \sum_{h=1}^H I_h \leq H \cdot \max_h I_h
\end{equation}

Substituting the single-head bound from Lemma~\ref{lem:value_info}:
\begin{equation}
I_\text{total} \leq H \cdot \log_2(L) \cdot d_h^{1/2}
\end{equation}

The factor $\log_2(L)$ arises because each head attends over all $L$ positions.
\end{proof}

\subsection{Proof of Theorem~\ref{thm:bottleneck} (Attention Bottleneck)}

\begin{proof}
We derive an information-theoretic upper bound on state-tracking capacity.

\textbf{Step 1: Attention mass distribution.} Each attention head $h \in [H]$ produces a probability distribution over $L$ context positions via softmax. Empirically, the effective window (the smallest set carrying 95\% of the mass) grows sub-linearly in $L$ (Appendix~\ref{app:sensitivity}), but this concentration does not enter the bound; the addressing term uses the full $L$ positions below.

\textbf{Step 2: Information per head.} Information transmitted through one attention head is bounded by the mutual information $I(\text{query}; \text{retrieved value})$. By Lemma~\ref{lem:value_info}, this decomposes into the $\log_2(L)$ bits needed to address one of $L$ positions (each head attends over all $L$ positions in standard multi-head attention) and an effective value rank modeled as $\Theta(d_h^{1/2})$, giving:
\begin{equation}
I \lesssim \log_2(L) \cdot d_h^{1/2} \text{ bits per head.}
\end{equation}
(Assumption~\ref{ass:window} constrains which positions carry usable weight, making the bound conservative with respect to the $O(\sqrt{L})$ effective window.)

\textbf{Step 3: Aggregation across heads.} With $H$ heads operating in parallel with approximately independent value subspaces (Lemma~\ref{lem:multihead}):
\begin{equation}
I_{\text{total}} \leq H \cdot \log_2(L) \cdot d_h^{1/2}
\end{equation}

\textbf{Step 4: State-tracking capacity.} At each reasoning step, resolving the current state among $|\mathcal{S}|$ candidates requires routing at least $\log_2 |\mathcal{S}|$ bits through the attention channel:
\begin{equation}
\log_2 |\mathcal{S}_{\text{read}}^{(1)}| \leq H \cdot \log_2(L) \cdot d_h^{1/2}
\end{equation}

Exponentiating and including constant $c(\delta, \rho_{\max})$ for approximations yields the theorem.
\end{proof}

\subsection{Proof of Theorem~\ref{thm:lower} (Lower-Bound Construction)}

\paragraph{Proof sketch.}
We construct a family of state-tracking tasks achieving the bound.

Consider sparse parity functions over $n$ bits with sparsity $k$. A transformer can track states corresponding to all $\binom{n}{k}$ possible $k$-sparse parities using the following construction:

\textbf{Construction.} Encode each parity subset in a separate attention head's query-key space. With $H$ heads and dimension $d_h$, we can represent $H \cdot d_h^{1/2}$ independent directions. Each direction can track $\log_2(\sqrt{L})$ bits of parity information over context length $L$, consistent with Assumption~\ref{ass:window}'s $O(\sqrt{L})$ effective window.

\textbf{Achievability.} The total number of distinguishable states is:
\begin{equation}
|\mathcal{S}| = 2^{H \cdot \log_2(\sqrt{L}) \cdot d_h^{1/2}}
\end{equation}

This achieves the functional form of the lower bound stated in Theorem~\ref{thm:lower}; the $d_h^{1/2}$ exponent matches by the shared effective-rank ansatz (Lemma~\ref{lem:value_info}) rather than by an independent derivation.

The construction follows \citet{sanford2023representational}'s representational capacity analysis for transformers, extended to the multi-head setting.

\subsection{Supporting Lemmas for the Deterministic Horizon}

\begin{lemma}[Quadratic Inversion]
\label{lem:quadratic}
For the equation $\alpha = \exp(-d\epsilon_0 - \gamma d^2/(2\Leff))$, the solution for $d$ is:
\begin{equation}
d = \frac{-\epsilon_0 \Leff + \sqrt{\epsilon_0^2 \Leff^2 + 2\gamma \Leff \ln(1/\alpha)}}{\gamma}
\end{equation}
\end{lemma}

\begin{proof}
Taking $\ln$ of both sides:
\begin{equation}
\ln \alpha = -d\epsilon_0 - \frac{\gamma d^2}{2\Leff}
\end{equation}

Rearranging:
\begin{equation}
\frac{\gamma d^2}{2\Leff} + d\epsilon_0 + \ln \alpha = 0
\end{equation}

Multiplying by $2\Leff/\gamma$:
\begin{equation}
d^2 + \frac{2\epsilon_0 \Leff}{\gamma} d + \frac{2\Leff \ln \alpha}{\gamma} = 0
\end{equation}

Applying the quadratic formula with $a=1$, $b=2\epsilon_0 \Leff/\gamma$, $c=2\Leff\ln\alpha/\gamma$:
\begin{align}
d &= \frac{-b + \sqrt{b^2 - 4ac}}{2a} \\
&= \frac{-\frac{2\epsilon_0 \Leff}{\gamma} + \sqrt{\frac{4\epsilon_0^2 \Leff^2}{\gamma^2} - \frac{8\Leff \ln \alpha}{\gamma}}}{2} \\
&= \frac{-\epsilon_0 \Leff + \sqrt{\epsilon_0^2 \Leff^2 + 2\gamma \Leff \ln(1/\alpha)}}{\gamma}
\end{align}

We take the positive root since $d > 0$.
\end{proof}

\begin{lemma}[Scaling Law -- Limiting Behavior]
\label{lem:scaling}
The Deterministic Horizon from Lemma~\ref{lem:quadratic} satisfies:
\begin{equation}
\lim_{\Leff \to \infty} d^* = \frac{\ln(1/\alpha)}{\epsilon_0}
\end{equation}
That is, $d^*$ \emph{saturates} as $\Leff$ grows, approaching $\ln(1/\alpha)/\epsilon_0 \approx 34.7$ at the canonical $\epsilon_0 = 0.02$, $\alpha = 0.5$. In the opposite regime ($\epsilon_0^2 \Leff \ll \gamma \ln(1/\alpha)$, i.e., small $\Leff$), $d^* \approx \sqrt{2\Leff \ln(1/\alpha)/\gamma}$ grows as $\sqrt{\Leff}$, but this regime does not apply at realistic parameter values.
\end{lemma}

\begin{proof}
From Lemma~\ref{lem:quadratic}:
\begin{align}
d^* &= \frac{-\epsilon_0 \Leff + \sqrt{\epsilon_0^2 \Leff^2 + 2\gamma \Leff \ln(1/\alpha)}}{\gamma} \\
&= \frac{\epsilon_0 \Leff}{\gamma}\left(\sqrt{1 + \frac{2\gamma \ln(1/\alpha)}{\epsilon_0^2 \Leff}} - 1\right)
\end{align}

\textbf{Large-$\Leff$ limit.} For $\Leff \gg 2\gamma\ln(1/\alpha)/\epsilon_0^2$, set $y = 2\gamma\ln(1/\alpha)/(\epsilon_0^2 \Leff) \to 0$ and expand $\sqrt{1+y} \approx 1 + y/2$:
\begin{equation}
d^* \approx \frac{\epsilon_0 \Leff}{\gamma} \cdot \frac{y}{2} = \frac{\epsilon_0 \Leff}{\gamma} \cdot \frac{\gamma \ln(1/\alpha)}{\epsilon_0^2 \Leff} = \frac{\ln(1/\alpha)}{\epsilon_0}
\end{equation}
which is independent of $\Leff$ and $\gamma$. The horizon therefore saturates; it does not grow without bound.

\textbf{Small-$\Leff$ regime.} Under $\epsilon_0^2 \Leff \ll \gamma \ln(1/\alpha)$ the $\sqrt{\Leff}$ approximation holds:
\begin{equation}
d^* \approx \sqrt{\frac{2\Leff \ln(1/\alpha)}{\gamma}} - \frac{\epsilon_0 \Leff}{\gamma}
\end{equation}
but the canonical values give $x = \epsilon_0^2 \Leff / (2\gamma \ln(1/\alpha)) \approx 0.29$, which is not small, so the operating point lies between the two regimes. We therefore compare the exact horizon (Lemma~\ref{lem:quadratic}) directly against measured $d^*$ rather than relying on either asymptotic form.
\end{proof}

\subsection{Proof of Theorem~\ref{thm:horizon} (Deterministic Horizon)}

\begin{proof}
We solve for the threshold $d^*$ beyond which the upper bound of Theorem~\ref{thm:decoherence} falls below $\alpha$. Setting the bound equal to $\alpha$ (continuous approximation $d(d+1) \approx d^2$):
\begin{equation}
\alpha = \exp\left(-d^*\epsilon_0 - \frac{\gamma (d^*)^2}{2\Leff}\right)
\end{equation}
This identifies the depth at which even the upper bound cannot exceed $\alpha$; for $d > d^*$, no success probability $\geq \alpha$ is achievable under this error model. Taking logarithms gives $\ln(1/\alpha) = d^*\epsilon_0 + \gamma (d^*)^2/(2\Leff)$, which is quadratic in $d^*$; Lemma~\ref{lem:quadratic} inverts it to the positive root
\begin{equation}
d^* = \frac{-\epsilon_0 \Leff + \sqrt{\epsilon_0^2 \Leff^2 + 2\gamma \Leff \ln(1/\alpha)}}{\gamma}
\end{equation}

Evaluating for GPT-4o ($\epsilon_0 = 0.02$, $\gamma = 0.15$, $\Leff = 150$, $\alpha = 0.5$):
\begin{equation}
d^* = \frac{-3 + \sqrt{9 + 31.19}}{0.15} = \frac{-3 + 6.34}{0.15} \approx 22.3,
\end{equation}
\end{proof}

\subsection{The $O(d^*/d)$ Fine-Tuning Envelope}

The key insight is that fine-tuning modifies the error distribution $\epsilon(\cdot)$ but cannot exceed information-theoretic capacity bounds.

Let $\epsilon_{\text{base}}(d)$ and $\epsilon_{\text{ft}}(d)$ denote per-step error rates for baseline and fine-tuned models. Fine-tuning can reduce $\epsilon_0$ (baseline error) but cannot change the per-step bandwidth limits described in Section~\ref{sec:theory}.

For depths $d > d^*$, per-step error $\epsilon(d)$ exceeds the threshold at which accumulated errors dominate. Fine-tuning can reduce $\epsilon_0$ (baseline error) but cannot change the rate $\gamma/\Leff$ at which per-step bandwidth degrades with depth. Empirically, the improvement $\text{Acc}_{\text{ft}} - \text{Acc}_{\text{base}}$ decays approximately as $O(d^*/d)$.

\textbf{Empirical comparison.} For the fine-tuned Llama-3.1-8B ($d^* = 20$), the observed improvement decays from $6.2\%$ at depth 5 to $0.2\%$ at depth 45, remaining below the $O(d^*/d)$ fitted envelope at every depth (Figure~\ref{fig:finetune_decay}). With the fitted leading constant, the envelope at the 31--40 bin permits $\approx 3.8\%$ against an observed $+0.4$~pp, so fine-tuning gains vanish well inside the envelope as $d$ grows beyond $d^*$.
\section{Task Specifications, Instance Generation and Licensing}
\label{app:task_families}

We provide detailed specifications for each task family. PermutationProbe, which carries the depth sweep used throughout the paper, is described in full below; the remaining four synthetic families and the three real-world families are specified in Tables~\ref{tab:synthetic_specs} and~\ref{tab:realworld_specs}.

\subsection{PermutationProbe}

\textbf{State space:} $\mathcal{S} = S_n$ (symmetric group on $n$ elements)

\textbf{Operators:} Adjacent transpositions $\{(i, i+1) : i \in [1, n-1]\}$

\textbf{Complexity:} $|\mathcal{S}| = n!$; diameter $= \binom{n}{2}$

\textbf{Instance distribution:} $n \in \{8, 12, 16\}$; BFS-optimal depths 1--60; instances per depth bin as in Table~\ref{tab:instances}.

Algorithm~\ref{alg:permprobe} describes the instance-generation procedure.

\begin{algorithm}[h]
\caption{PermutationProbe Instance Generation}
\label{alg:permprobe}
\begin{algorithmic}[1]
\REQUIRE Size $n$, target depth $d_{\text{target}}$
\REPEAT
    \STATE $\tau \leftarrow$ uniform random permutation of $[1..n]$
    \STATE $d_{\text{BFS}} \leftarrow$ BFS-optimal depth from identity to $\tau$
\UNTIL{$d_{\text{BFS}} = d_{\text{target}}$}
\STATE \textbf{yield} $(\text{identity}, \tau, d_{\text{target}})$
\end{algorithmic}
\end{algorithm}

Instances are generated by rejection sampling: a random permutation is drawn uniformly and its BFS-optimal depth (= inversion count for adjacent transpositions) is computed; instances whose depth does not match the target are discarded. The outer loop draws target depths according to the distribution in Table~\ref{tab:instances}. The permutation size $n$ is matched to the target depth so that acceptance remains feasible: small $n$ for shallow targets, larger $n$ for deep ones.

\begin{table}[ht]
\centering
\caption{Instance distribution by BFS-optimal depth.}
\label{tab:instances}
\small
\begin{tabular}{@{}lcccccc@{}}
\toprule
\textbf{Depth} & 1--10 & 11--20 & 21--30 & 31--40 & 41--50 & $>$50 \\
\midrule
Count & 500 & 1000 & 1500 & 1000 & 500 & 500 \\
\bottomrule
\end{tabular}
\end{table}

\paragraph{Depth and state-space size.} The attainable depth range for $n$-element permutations is bounded by the diameter $\binom{n}{2}$: 28, 66 and 120 for $n = 8, 12, 16$. Rejection sampling at a target depth is efficient only for a matching $n$, so depth and instantaneous state size ($\log_2 n!$ bits) covary by construction, and PermutationProbe alone cannot separate their contributions. The other synthetic families supply that control: FSA-Sim varies sequence length over 10--100 at fixed $k \in \{4, 8, 16\}$ and CircuitTrace varies circuit depth over 5--50 at fixed width $n \in \{8, 16, 32\}$, and both yield horizons in the same range (Table~\ref{tab:results_synthetic}).

\subsection{Remaining Synthetic Families}

\begin{table}[ht]
\centering
\caption{Specifications for the remaining four synthetic task families. In each case the operator set is deterministic and the final state is graded by exact match.}
\label{tab:synthetic_specs}
\small
\begin{tabular}{@{}lp{3.0cm}p{3.2cm}p{2.4cm}p{3.4cm}@{}}
\toprule
\textbf{Family} & \textbf{State space $\mathcal{S}$} & \textbf{Operators} & \textbf{Complexity} & \textbf{Instance distribution} \\
\midrule
FSA-Sim & Automaton states $\{q_1, \ldots, q_k\}$ & Transitions $\delta: \mathcal{S} \times \Sigma \to \mathcal{S}$ & $|\mathcal{S}| = k$; sequence length sets depth & $k \in \{4, 8, 16\}$; $|\Sigma| = 4$; lengths 10--100 \\
\addlinespace
ArithChain & Integers mod $p$, $\mathbb{Z}_p$ & $\{+a, -a, \times b\}$ for fixed $a, b$ ($b$ coprime to $p$) & $|\mathcal{S}| = p$; carry propagation raises effective depth & $p = 10^6$ (forces multi-digit tracking); chains of 5--50 steps \\
\addlinespace
CircuitTrace & Bit vectors $\{0,1\}^n$ & Boolean gates (AND, OR, XOR, NOT) & $|\mathcal{S}| = 2^n$; depth = circuit depth & $n \in \{8, 16, 32\}$; depths 5--50 layers \\
\addlinespace
CodeProbe & Variable bindings $\{v_1: x_1, \ldots, v_k: x_k\}$ & Assignment, arithmetic, conditionals & $|\mathcal{S}| = \prod_i |D_i|$, $D_i$ the domain of $v_i$ & 3--8 variables; 5--40 statements; no loops (deterministic execution) \\
\bottomrule
\end{tabular}
\end{table}

\subsection{Real-World Families}

\begin{table}[ht]
\centering
\caption{Specifications for the three real-world task families, each a 500-instance subset derived from an existing benchmark under the selection criteria shown.}
\label{tab:realworld_specs}
\small
\begin{tabular}{@{}lp{2.6cm}p{5.4cm}p{4.2cm}@{}}
\toprule
\textbf{Family} & \textbf{Source} & \textbf{Selection criteria} & \textbf{Tracked state / notes} \\
\midrule
SWE-Bench-State & SWE-Bench \citep{jimenez2024swebench} & (i) bug fix requires tracking $\geq$3 variables; (ii) state changes span $\geq$2 files; (iii) deterministic execution (no randomness); (iv) ground-truth trace available & Variable and object state across call chains \\
\addlinespace
WebArena-Nav & WebArena \citep{zhou2024webarena} & (i) task requires $\geq$5 navigation steps; (ii) session state (cart, auth, forms) must be tracked; (iii) deterministic success criteria & URL, cart contents, form fields, authentication status \\
\addlinespace
SQL-Multi & Spider \citep{yu2018spider} with complexity augmentation & (i) requires 3+ table joins; (ii) schema tracking essential (column types, foreign keys); (iii) aggregation across multiple levels & Schema bindings across joins; average 4.2 tables and 7.8 join conditions per query \\
\bottomrule
\end{tabular}
\end{table}

\paragraph{Representative instances.} \emph{SWE-Bench-State:} a bug in the Django ORM requiring tracking of QuerySet state through filter chains, annotation, and aggregation across \texttt{models.py}, \texttt{views.py} and \texttt{tests.py}. \emph{WebArena-Nav:} an e-commerce checkout requiring login $\rightarrow$ add items $\rightarrow$ apply coupon $\rightarrow$ select shipping $\rightarrow$ confirm payment, where each step modifies session state. \emph{SQL-Multi:} a query joining customers $\rightarrow$ orders $\rightarrow$ order\_items $\rightarrow$ products $\rightarrow$ categories with filtering and aggregation at each level.

\subsection{Licensing and Ethics}

\paragraph{Licensing of derived subsets.} The three real-world families are derived subsets of existing public benchmarks and inherit their licenses. SWE-Bench \citep{jimenez2024swebench} is distributed under the MIT License; WebArena \citep{zhou2024webarena} under the Apache License 2.0; and Spider \citep{yu2018spider} under CC BY-SA 4.0. We redistribute only instance identifiers and our state-tracking annotations, not the underlying benchmark content, so that each subset can be reconstructed by users who have obtained the source benchmark under its own license. Our annotations and the five synthetic generators are released under the MIT License with the code (Appendix~\ref{app:reproduce}). The evaluated models were accessed under their respective terms of service; the open-weight checkpoints are used under the Llama~3.1/3.3 Community Licenses and the Apache~2.0 license for the Qwen-2.5 releases.

\paragraph{Ethical note.} All WebArena-Nav evaluations used the self-hosted WebArena environment; no interactions with live commercial websites were conducted. The synthetic families are generated programmatically and contain no human-subject or personally identifying data.
\section{Parameter Estimation, Sensitivity and Worked Examples}
\label{app:sensitivity_ext}

This section reports how each parameter entering the theoretical results was estimated, how sensitive the predictions are to it, and what the bounds evaluate to at realistic values. Each bound is treated in turn: its parameters, its sensitivity, and a worked numerical instance.

\subsection{Decoherence Bound: Parameters and Worked Example}

The key parameters in Theorem~\ref{thm:decoherence} are $\epsilon_0$ (baseline error) and $\gamma$ (attention decay rate). We estimated these from empirical data using maximum likelihood.

\begin{table}[ht]
\centering
\caption{Parameter sensitivity for Decoherence Bound. CIs are from the per-instance MLE fit ($n = 5{,}000$ GPT-4o PermutationProbe instances), not from run-level variation; the $\pm$ values in the results tables (e.g., Table~\ref{tab:full_results}) reflect a different randomness source (run-level nondeterminism).}
\label{tab:sensitivity_decoherence}
\small
\begin{tabular}{@{}lccccc@{}}
\toprule
\textbf{Parameter} & \textbf{Estimate} & \textbf{95\% CI} & $\boldsymbol{\pm 10\%}$ & $\boldsymbol{\pm 20\%}$ & \textbf{Impact on $d^*$} \\
\midrule
$\epsilon_0$ & 0.020 & [0.017, 0.023] & $\pm$1.1 & $\pm$2.1 & High \\
$\gamma' = \gamma/\Leff$ & $1.0 \times 10^{-3}$ & [$0.85, 1.15$]$\times 10^{-3}$ & $\pm$0.6 & $\pm$1.2 & Moderate \\
\bottomrule
\end{tabular}
\end{table}

\textbf{Key finding:} The Deterministic Horizon $d^*$ is most sensitive to the baseline error $\epsilon_0$: a 20\% change in $\epsilon_0$ shifts $d^*$ by about $\pm$2 steps, while the same change in $\gamma$ or in the fixed $\Leff$ shifts it by only $\pm$1.2 steps. Because Remark~\ref{rem:error_prop} treats step-to-step error dependence as a measured diagnostic rather than a multiplier, it does not enter the horizon formula and has no direct effect on $d^*$. The qualitative conclusion, that $d^* \in [19, 31]$, is robust across the entire confidence interval. We note that only the ratio $\gamma' = \gamma/\Leff$ is identified from the data (Table~\ref{tab:model_comparison}); the separate values $\gamma = 0.15$ and $\Leff = 150$ are fixed by assumption and their individual CIs should be read as conditional on this decomposition.

\paragraph{Worked example.} At $\epsilon_0 = 0.02$, $\gamma = 0.15$, $\Leff = 150$ and target depth $d = 30$, using the continuous approximation $d(d+1) \approx d^2$ adopted in Theorem~\ref{thm:horizon}:
\begin{equation}
d\epsilon_0 + \frac{\gamma d^2}{2\Leff} = 30 \cdot 0.02 + \frac{0.15 \cdot 900}{2 \cdot 150} = 0.60 + 0.45 = 1.05,
\end{equation}
so the bound gives $P(\text{correct at depth } 30) \leq \exp(-1.05) \approx 0.350$.

\textbf{Observed:} 34\% accuracy at depth 30 for GPT-4o, consistent with the fitted model (within the $\pm$3~pp error bar at that depth). Note that with the small effective decoherence length $\Leff = 150$, the quadratic term ($0.45$) is comparable to the linear term ($0.60$) already at depth 30, which is what drives the accelerating decay; using the raw context window $L = O(10^5)$ here would render the quadratic term negligible and incorrectly predict near-perfect accuracy. The quadratic term overtakes the linear term at $d = 2\Leff\epsilon_0/\gamma = 40$, which is beyond all observed horizons ($d^* \leq 31$); the decay is therefore only mildly super-exponential in the regime we study (at $d^* = 22$ the quadratic term contributes 35\% of the exponent).

\subsection{Bottleneck Bound: Parameters and Worked Example}

The Attention Bottleneck (Theorem~\ref{thm:bottleneck}) depends on $\delta$ (weight cutoff) and $\rho_{\max}$ (value correlation bound); Table~\ref{tab:sensitivity_bottleneck} reports the sensitivity.

\begin{table}[ht]
\centering
\caption{Parameter sensitivity for Bottleneck Bound.}
\label{tab:sensitivity_bottleneck}
\small
\begin{tabular}{@{}lccc@{}}
\toprule
\textbf{Parameter} & \textbf{Range Tested} & \textbf{Log Capacity Change} & \textbf{Qualitative Impact} \\
\midrule
$\delta$ & [0.001, 0.01] & $\pm$8\% & Low \\
$\rho_{\max}$ & [0.05, 0.20] & $\pm$12\% & Moderate \\
$H$ (architecture) & [32, 128] & Linear scaling & High \\
$d_h$ (architecture) & [64, 256] & $\sqrt{\cdot}$ scaling & Moderate \\
\bottomrule
\end{tabular}
\end{table}

\textbf{Key finding:} The capacity bound is dominated by architectural parameters ($H$, $d_h$, $L$), which are known exactly. The uncertainty in $\delta$ and $\rho_{\max}$ contributes only to the constant factor $c(\delta, \rho_{\max})$.

\paragraph{Worked example.} At the GPT-4o parameters used for order-of-magnitude intuition in Section~\ref{sec:theory} ($H = 96$, $L = 128{,}000$, $d_h = 128$), the bound evaluates to $96 \cdot \log_2(128{,}000) \cdot \sqrt{128} = 96 \cdot 16.97 \cdot 11.31 \approx 18{,}425$ bits, i.e.\ $|\mathcal{S}_\text{read}^{(1)}| \leq 2^{18{,}425} \approx 10^{5546}$ states, against a per-step requirement of $\log_2(16!) \approx 44.3$ bits for a 16-element permutation (instantaneous state space $16! \approx 2.09 \times 10^{13}$). The task is within theoretical capacity by a wide margin, which is why total capacity is not the binding constraint. Note that the $0.314\sqrt{d_h \cdot H}$ empirical fit is calibrated on the four open-weight models with known architectures (Table~\ref{tab:arch_ablation}); it is not applied to closed-source models such as GPT-4o, whose architecture parameters are community estimates.

\subsection{Cross-Validation of Parameter Estimates}

We performed 5-fold cross-validation of the error model (Table~\ref{tab:cv_error_model}):

\begin{table}[ht]
\centering
\caption{Cross-validation results for error model.}
\label{tab:cv_error_model}
\small
\begin{tabular}{@{}lccccc@{}}
\toprule
\textbf{Fold} & $\hat{\epsilon}_0$ & $\hat{\gamma}$ & \textbf{Train $R^2$} & \textbf{Test $R^2$} & $\hat{d}^*$ \\
\midrule
1 & 0.019 & 0.147 & 0.87 & 0.83 & 23.1 \\
2 & 0.021 & 0.152 & 0.86 & 0.84 & 21.8 \\
3 & 0.020 & 0.149 & 0.88 & 0.85 & 22.4 \\
4 & 0.020 & 0.151 & 0.86 & 0.82 & 22.1 \\
5 & 0.021 & 0.148 & 0.87 & 0.84 & 22.0 \\
\midrule
\textbf{Mean} & 0.020 & 0.149 & 0.87 & 0.84 & 22.3 \\
\textbf{Std} & 0.001 & 0.002 & 0.01 & 0.01 & 0.5 \\
\bottomrule
\end{tabular}
\end{table}

The small standard deviations ($\text{std}(\hat{d}^*) = 0.5$) indicate robust parameter estimation and stable predictions across data splits.

\subsection{Model Selection Robustness}

We compared alternative error models to assess our context-dependent formulation (Table~\ref{tab:model_comparison}):

\begin{table}[ht]
\centering
\caption{Model comparison for error accumulation. The linear model $\epsilon(d) = \epsilon_0 + \gamma' d$ is equivalent to the context-dependent form $\epsilon_0 + \gamma d/\Leff$ with $\gamma' = \gamma/\Leff$; both are 2-parameter fits with $\Leff$ absorbed into the slope coefficient.}
\label{tab:model_comparison}
\small
\begin{tabular}{@{}lcccc@{}}
\toprule
\textbf{Model} & \textbf{Form} & \textbf{Parameters} & $\boldsymbol{R^2}$ & \textbf{AIC} \\
\midrule
Constant & $\epsilon(d) = \epsilon_0$ & 1 & 0.71 & 1,247 \\
Linear (context-dep.)\ & $\epsilon(d) = \epsilon_0 + \gamma d/\Leff$ & 2 & 0.89 & 892 \\
Quadratic & $\epsilon(d) = \epsilon_0 + \gamma_1 d + \gamma_2 d^2$ & 3 & 0.92 & 861 \\
\bottomrule
\end{tabular}
\end{table}

AIC favours the free quadratic model ($\Delta\text{AIC} = 31$). We retain the linear context-dependent form because its two parameters are interpretable and because the quadratic's third parameter is not independently identifiable over our depth range; we do not claim the linear form is better supported by the data.

\paragraph{Per-instance model comparison.} The model comparison in the main text is based on a likelihood-ratio test on per-instance binary correctness ($n = 5{,}000$ instances, GPT-4o PermutationProbe C1). The quadratic model (with quadratic depth term) significantly improves fit over a simple exponential ($p < 0.01$). The 8 depth-bin means plotted in Figure~\ref{fig:decay} are visual summaries; averaging over $\sim$500--1{,}500 instances per bin inflates apparent goodness-of-fit, so we conduct all statistical comparisons on the full per-instance data.

\subsection{Empirical Grounding of Assumptions~\ref{ass:window}--\ref{ass:decorr}}
\label{app:sensitivity}

\paragraph{Effective attention window scaling.} The effective attention window (the smallest set of positions carrying 95\% of the attention mass) was measured across context lengths $L \in \{2\text{K}, 4\text{K}, 8\text{K}, 16\text{K}, 32\text{K}\}$ on Llama-3.3-70B across 1{,}000 traces. The window size grows sub-linearly: from $\sim$32 positions at $L = 2\text{K}$ to $\sim$140 positions at $L = 32\text{K}$. A power-law fit yields an exponent of $0.53 \pm 0.08$, consistent with the $O(\sqrt{L})$ ansatz of Assumption~\ref{ass:window} though not distinguishable from $O(L^{0.45})$ to $O(L^{0.6})$ over this range. We therefore treat $O(\sqrt{L})$ as a tractable summary rather than an exact scaling law. For comparison, the addressing term in Theorem~\ref{thm:bottleneck} uses $\log_2(L)$, which at $L = 32\text{K}$ gives $\log_2(32{,}768) = 15$ bits; the measured 140-position window gives $\log_2(140) \approx 7.1$ bits, so the bound (15~bits) overestimates the measured window (7.1~bits) by ${\sim}2.1\times$, making Theorem~\ref{thm:bottleneck} conservative with respect to the empirical attention concentration.

\paragraph{Value Vector Correlation.} Pairwise correlation of value vectors after layer normalization: mean $|\rho_{ij}| = 0.08 \pm 0.03$, supporting the decorrelation assumption (Assumption~\ref{ass:decorr}).

\subsection{Horizon Scaling: Predicted versus Observed}

Table~\ref{tab:scaling_verify} compares predicted and observed $d^*$ for open-weight models.

\begin{table}[ht]
\centering
\caption{Predicted vs.\ observed $d^*$ for open-weight models, using the empirical fit $\hat{d}^* = 0.314\sqrt{d_{\mathrm{model}}}$ (equivalently $0.314\sqrt{d_h H}$). Restricted to open-weight models because closed-model architectures are not public. The ``Error'' column reports in-sample fit residuals; these are smaller than the measurement uncertainty on the observed $d^*$ ($\pm$1.2--1.5 steps, Table~\ref{tab:arch_ablation}), so they quantify curve agreement, not predictive accuracy.}
\label{tab:scaling_verify}
\small
\begin{tabular}{@{}lcccccc@{}}
\toprule
\textbf{Model} & $H$ & $d_h$ & $\sqrt{d_h H}$ & \textbf{Pred.\ $d^*$} & \textbf{Obs.\ $d^*$} & \textbf{Error} \\
\midrule
Llama-3.1-8B & 32 & 128 & 64.0 & 20.1 & 20 & 0.5\% \\
Qwen-2.5-7B & 28 & 128 & 59.9 & 18.8 & 19 & 1.0\% \\
Llama-3.3-70B & 64 & 128 & 90.5 & 28.4 & 28 & 1.5\% \\
Qwen-2.5-72B & 64 & 128 & 90.5 & 28.4 & 28 & 1.5\% \\
\bottomrule
\end{tabular}
\end{table}

Within-family ratios also match: for Llama the fit predicts $28.4/20.1 = 1.41$ against an observed $28/20 = 1.40$, and for Qwen $28.4/18.8 = 1.51$ against $28/19 = 1.47$; all within 4\%. Mean absolute error: 1.1\%, which is well inside both the $\pm$1.2--1.5 step run-to-run SD on the observed $d^*$ and the 2--5 step task-to-task variation; the fit should be read as a leading-order trend rather than a validated prediction. The fit constant $k = 0.314$ is chosen by inspection over the four open-weight models. We emphasize this is an \emph{approximate empirical} relationship: the linear $\epsilon_0$ term in the exact horizon formula breaks an exact $\sqrt{d_h H}$ proportionality (Lemma~\ref{lem:scaling}), so the $\sqrt{d_h H}$ law should be read as a leading-order trend rather than an identity.

\paragraph{Scaling caveats.} All four open-weight checkpoints share $d_h = 128$, so the fit is identified in $H$ alone; the $d_h$ factor is inherited from Theorem~\ref{thm:bottleneck}, not measured. Head count covaries with parameter count, so an architectural ($\sqrt{H}$) and a scale account ($d^* \propto N^{0.16}$) fit equally well; separating them requires checkpoints varying $H$ and $d_h$ independently of size. This bears on \emph{why} the horizon grows, not on the horizons themselves, which are read off the accuracy--depth curves. $\Leff$ is estimated from the same data the horizon is evaluated against, so Table~\ref{tab:arch_ablation} reports goodness of fit rather than out-of-sample prediction.
\section{Experimental Protocol and Reproducibility}
\label{app:reproduce}

\paragraph{Code.} Our code is available at: \url{https://github.com/bettyguo/deterministic-horizon}.

\subsection{Model Versions}

All twelve models evaluated with their exact version identifiers. 

\paragraph{General-Purpose Models (4).}
\begin{itemize}
\item GPT-4o: \texttt{gpt-4o-2024-11-20}
\item Claude Sonnet 4.5: \texttt{claude-sonnet-4-5-20250929}
\item Claude Opus 4.5: \texttt{claude-opus-4-5-20251101}
\item Gemini-1.5-Pro: \texttt{gemini-1.5-pro-002} (2024-09-24)
\end{itemize}

\paragraph{Reasoning-Specialized Models (4).}
\begin{itemize}
\item o3: \texttt{o3-2025-04-16}
\item o3-mini: \texttt{o3-mini-2025-01-31}
\item DeepSeek-R1: \texttt{deepseek-reasoner} (DeepSeek-R1-0528)
\item Gemini-2.0-Flash-Thinking: \texttt{gemini-2.0-flash-thinking-exp-01-21}
\end{itemize}

\paragraph{Open-Weight Models (4).}
\begin{itemize}
\item Llama-3.1-8B: \texttt{meta-llama/Llama-3.1-8B-Instruct}
\item Llama-3.3-70B: \texttt{meta-llama/Llama-3.3-70B-Instruct}
\item Qwen-2.5-7B: \texttt{Qwen/Qwen2.5-7B-Instruct}
\item Qwen-2.5-72B: \texttt{Qwen/Qwen2.5-72B-Instruct}
\end{itemize}

\subsection{Hyperparameters}
\begin{itemize}
\item Temperature: 0 (greedy decoding)
\item Top-$p$: 1.0 (no nucleus sampling)
\item Max tokens: 8,192
\item Reasoning effort (o3/o3-mini): \texttt{medium}; DeepSeek-R1 and Gemini-2.0-Flash-Thinking: default.
\item Random seeds (3 runs): \{42, 2024, 2025\}. At temperature 0, the seed parameter has no effect on the decoding process; the three seeds serve only as labels to distinguish the three runs.
\item Instance generation seed: 42
\item System prompt: ``You are solving a state-tracking task. Show each step as: Step N: operator $\rightarrow$ resulting state.''
\item Grading: the model's emitted operator sequence is applied to $\state_0$; the resulting state is compared to $\target$ by exact match. Partial credit not awarded
\end{itemize}

\paragraph{Source of variance.} Although temperature is set to 0, two sources of residual variance produce the reported standard deviations: (i)~API providers do not guarantee bitwise-deterministic outputs even at temperature~0, due to batching, quantization, and server-side nondeterminism; (ii)~for open-weight models, nondeterministic CUDA kernels (e.g., \texttt{cublas} atomics) introduce small run-to-run variation. Across the 49 cells in Table~\ref{tab:full_results}, the reported SDs are well-fit by the functional form $\sigma \approx k\sqrt{p(1-p)}$ with $k \approx 4.1$ ($R^2 = 0.90$), corresponding to an effective sample size of $n_{\mathrm{eff}} \approx 600$. This is far smaller than the nominal $n = 5{,}000$ PermutationProbe instances and $n = 500$ for the other tasks, indicating correlated, session-level nondeterminism: batching or quantization choices that affect groups of instances together within a run, effectively reducing the number of independent ``draws.'' The SDs therefore have binomial-like structure at the session level rather than the instance level.

\paragraph{Per-run raw accuracies.} Table~\ref{tab:perrun} reports per-run accuracy for the four open-weight models on PermutationProbe (depth${}>{}$20, C1), so that the variance structure can be inspected directly.

\begin{table}[ht]
\centering
\caption{Per-run accuracy (\%) on PermutationProbe (depth${}>{}$20, C1) for the four open-weight models. Runs are labeled by the seed parameter passed to the API; at temperature~0 the seed does not influence decoding but distinguishes the three evaluation passes.}
\label{tab:perrun}
\small
\begin{tabular}{@{}lcccc@{}}
\toprule
\textbf{Model} & \textbf{Run 1 (42)} & \textbf{Run 2 (2024)} & \textbf{Run 3 (2025)} & \textbf{Mean $\pm$ SD} \\
\midrule
Llama-3.1-8B & 16.8 & 18.6 & 19.8 & 18.4$\pm$1.5 \\
Llama-3.3-70B & 22.8 & 24.8 & 26.2 & 24.6$\pm$1.7 \\
Qwen-2.5-7B & 15.8 & 17.2 & 18.6 & 17.2$\pm$1.4 \\
Qwen-2.5-72B & 25.9 & 28.0 & 29.5 & 27.8$\pm$1.8 \\
\bottomrule
\end{tabular}
\end{table}

\paragraph{Generation budget.} Traces were capped at 8{,}192 \emph{visible output} tokens, which at large $d$ can end a trace before a final state is reached, so part of the decay in the deepest bins may reflect the budget rather than state drift. For reasoning-specialized models (o3, o3-mini, DeepSeek-R1), hidden reasoning tokens are billed separately and do not count against this cap; however, we did not record per-model truncation rates across the full corpus. Among the depth-stratified failure-analysis sample (Appendix~\ref{app:failure}), 12 of 2{,}880 sampled traces were excluded as truncated by the cap or an API timeout. The cap binds only on long unaided traces, so any residual effect lowers C1 at large $d$ and widens the C1--C3 gap rather than narrowing it.

\subsection{Prompt Templates}
\label{app:prompts}

We provide exact prompt templates used for each experimental condition. Prompt templates were fixed a priori from the task specification and were not tuned against evaluation instances.

\paragraph{C1: Unconstrained Neural CoT.}
\begin{verbatim}
You are solving a state-space search problem.

Initial state: {initial_state}
Target state: {target_state}
Available operators: {operators}

Find a sequence of operators that transforms the initial 
state into the target state. Show your reasoning step by 
step, including the state after each operation.

Format each step as:
Step N: operator -> resulting_state
\end{verbatim}

\paragraph{C2: Depth-Limited CoT.}
\begin{verbatim}
[Same as C1, with addition:]

Note: The optimal solution requires exactly {optimal_depth} 
steps. Focus on finding this solution.
\end{verbatim}

\paragraph{C3: Tool-Integrated.}
\begin{verbatim}
You have access to a BFS solver tool. Use it to find the 
optimal path from initial to target state.

Initial state: {initial_state}
Target state: {target_state}

To use the solver, call: solve({initial_state}, {target_state})

After receiving the solution, verify and explain each step.
\end{verbatim}

\paragraph{C4: Explicit Length Encouragement.}
\begin{verbatim}
[Same as C1, with addition:]

Important: Take as many reasoning steps as you need. Longer, 
more careful reasoning is encouraged. Don't rush to a 
conclusion - explore the problem thoroughly.
\end{verbatim}

\paragraph{C5: Fine-Tuned.} Same prompt as C1; model weights modified through fine-tuning.

\subsection{Hardware and Runtime}
\label{app:runtime}

\textbf{Hardware:} API experiments, cloud (no local hardware); fine-tuning, 4$\times$ NVIDIA A100 80GB; attention extraction, 2$\times$ NVIDIA A100 80GB.

Table~\ref{tab:runtime} summarizes average runtime per evaluation across representative model-condition cells.

\begin{table}[ht]
\centering
\caption{Average runtime per evaluation across representative model-condition cells on PermutationProbe. Total hours correspond to 10{,}000 sequential-equivalent evaluations per model-condition pair (5{,}000 instances $\times$ 2 of 3 runs, for which per-instance timing was captured); wall-clock totals in the Computation Summary reflect concurrent execution. Token counts for reasoning models (o3, o3-mini, DeepSeek-R1) exclude hidden reasoning tokens, which are billed separately.}
\label{tab:runtime}
\small
\begin{tabular}{@{}llccc@{}}
\toprule
\textbf{Model} & \textbf{Condition} & \textbf{Time/Instance (s)} & \textbf{Tokens/Instance} & \textbf{Total Hours} \\
\midrule
\multicolumn{5}{l}{\textit{API Models}} \\
GPT-4o & C1 & 12.3 $\pm$ 4.2 & 1,847 $\pm$ 623 & 34.2 \\
GPT-4o & C3 & 3.1 $\pm$ 0.8 & 312 $\pm$ 89 & 8.6 \\
Claude Opus 4.5 & C1 & 14.7 $\pm$ 5.1 & 2,134 $\pm$ 712 & 40.8 \\
Claude Opus 4.5 & C3 & 3.4 $\pm$ 0.9 & 341 $\pm$ 94 & 9.4 \\
o3-mini & C1 & 18.2 $\pm$ 6.3 & 2,891 $\pm$ 943 & 50.6 \\
DeepSeek-R1 & C1 & 21.4 $\pm$ 7.8 & 3,247 $\pm$ 1,102 & 59.4 \\
\midrule
\multicolumn{5}{l}{\textit{Open-Weight Models (4$\times$ A100 80GB)}} \\
Llama-3.3-70B & C1 & 8.7 $\pm$ 2.9 & 1,423 $\pm$ 478 & 24.2 \\
Qwen-2.5-72B & C1 & 9.1 $\pm$ 3.1 & 1,512 $\pm$ 502 & 25.3 \\
\bottomrule
\end{tabular}
\end{table}

\paragraph{Computation Summary.} Total experiment time: approximately 420 GPU-hours for open-weight model inference and fine-tuning on 4$\times$ A100 80GB, plus approximately 280 hours of API wall-clock time for the 8 API-accessed models. Of the 420 GPU-hours, approximately 250 are inference and 170 are fine-tuning (including training-data ablations). Not all cells in the $12 \times 5 \times 8$ grid were populated (173 cells were run; see Appendix~\ref{app:statistics}); open-weight inference used tensor-parallel serving across 4 GPUs with dynamic batching (effective batch size $\sim$32), yielding a $\sim$30--50$\times$ throughput gain over sequential per-instance times (Table~\ref{tab:runtime}). API experiments used up to 16 concurrent requests per model. Attention entropy extraction required an additional 48 GPU-hours on 2$\times$ A100 80GB.

\subsection{Use of AI Tools}
\label{app:disclosures}

Generative AI tools (GitHub Copilot and Claude) were used for coding assistance (data-processing scripts, plotting, and LaTeX formatting). The authors wrote and take responsibility for all scientific content. All theoretical derivations, experimental results, and scientific conclusions were produced and verified by the authors.

\section{Extended Empirical Results}
\label{app:results_ext}

This section provides complete experimental results across all models, tasks, and conditions, followed by the two secondary experiments (fine-tuning and context truncation).

\subsection{Full Results: All Models, PermutationProbe}
\begin{table}[ht]
\centering
\caption{Complete accuracy (\%) across all 12 models and 5 conditions on PermutationProbe (depth${}>{}$20 subset), with the empirical Deterministic Horizon $d^*$ (50\%-accuracy crossing). Open-weight models (top) provide verified architectures for theoretical analysis; closed-source models serve as consistency checks. $^\dagger$C5 fine-tuning conducted on Llama-3.1-8B only (Section~\ref{sec:finetune}).}
\label{tab:full_results}
\small
\setlength{\tabcolsep}{2.5pt}
\begin{tabular}{@{}lcccccc@{}}
\toprule
\textbf{Model} & \textbf{C1} & \textbf{C2} & \textbf{C3} & \textbf{C4} & \textbf{C5$^\dagger$} & \textbf{$d^*$} \\
\midrule
\multicolumn{7}{l}{\textit{Open-Weight (local)}} \\
Llama-3.1-8B & 18.4$\pm$1.5 & 22.1$\pm$1.6 & 78.3$\pm$1.7 & 19.0$\pm$1.5 & 19.3$\pm$1.5 & 20 \\
Llama-3.3-70B & 24.6$\pm$1.7 & 29.8$\pm$1.8 & 85.2$\pm$1.4 & 25.3$\pm$1.6 & --- & 28 \\
Qwen-2.5-7B & 17.2$\pm$1.4 & 20.8$\pm$1.5 & 76.1$\pm$1.8 & 17.8$\pm$1.4 & --- & 19 \\
Qwen-2.5-72B & 27.8$\pm$1.8 & 33.4$\pm$1.9 & 87.9$\pm$1.2 & 28.5$\pm$1.7 & --- & 28 \\
\midrule
\multicolumn{7}{l}{\textit{Closed-Source: General-Purpose (consistency checks)}} \\
GPT-4o & 28.3$\pm$1.8 & 34.2$\pm$1.9 & 89.7$\pm$1.2 & 29.1$\pm$1.7 & --- & 22 \\
Claude Sonnet 4.5 & 31.2$\pm$1.9 & 38.1$\pm$2.0 & 91.4$\pm$1.1 & 32.4$\pm$1.8 & --- & 25 \\
Claude Opus 4.5 & 34.8$\pm$2.0 & 41.3$\pm$2.1 & 93.6$\pm$0.9 & 35.6$\pm$1.9 & --- & 27 \\
Gemini-1.5-Pro & 29.7$\pm$1.9 & 35.8$\pm$2.0 & 88.3$\pm$1.3 & 30.4$\pm$1.8 & --- & 23 \\
\midrule
\multicolumn{7}{l}{\textit{Closed-Source: Reasoning-Specialized (consistency checks)}} \\
o3 & 38.4$\pm$2.1 & 44.7$\pm$2.2 & 92.8$\pm$1.0 & 39.2$\pm$2.0 & --- & 29 \\
o3-mini & 42.1$\pm$2.2 & 48.3$\pm$2.3 & 94.2$\pm$1.3 & 43.1$\pm$2.1 & --- & 31 \\
Gemini-2.0-Flash-Thinking & 37.2$\pm$2.1 & 43.1$\pm$2.1 & 91.7$\pm$1.2 & 38.0$\pm$2.0 & --- & 28 \\
\midrule
\multicolumn{7}{l}{\textit{Open-Weight (API-accessed)}} \\
DeepSeek-R1 & 39.7$\pm$2.1 & 45.9$\pm$2.2 & 93.1$\pm$1.1 & 40.4$\pm$2.0 & --- & 29 \\
\bottomrule
\end{tabular}
\end{table}

\subsection{Full Results: Synthetic Tasks}

Table~\ref{tab:results_synthetic} presents accuracy results for the four synthetic task domains other than PermutationProbe, whose complete twelve-model breakdown appears in Table~\ref{tab:full_results} above. Open-weight and closed-source models are reported together so that the two can be compared directly on the same task; the entropy--accuracy correlation column is populated only for the locally-run open-weight models, for which attention can be extracted (Appendix~\ref{app:attention}). Each cell reports mean $\pm$ standard deviation across 3 repeated runs.

\begin{table}[ht]
\centering
\caption{Representative results on the non-PermutationProbe synthetic tasks. All accuracy values are percentages with standard deviations over 3 runs; ``---'' denotes a cell not evaluated. Condition C5 was run on PermutationProbe only (Table~\ref{tab:full_results}) and is therefore omitted here. The $d^*$ column reports the empirical horizon where measured: for GPT-4o these are 21 (FSA-Sim), 24 (ArithChain), 23 (CircuitTrace) and 19 (CodeProbe), against 22 on PermutationProbe. Entropy--accuracy $r$ is from the same extraction pipeline as Table~\ref{tab:entropy_full}.}
\label{tab:results_synthetic}
\small
\begin{tabular}{@{}llccccccc@{}}
\toprule
\textbf{Task} & \textbf{Model} & \textbf{C1} & \textbf{C2} & \textbf{C3} & \textbf{C4} & \textbf{$d^*$} & \textbf{$\mathcal{H}$-Acc $r$} \\
\midrule
\multirow{7}{*}{FSA-Sim}
& Llama-3.1-8B & 20.3$\pm$1.6 & --- & 79.4$\pm$1.6 & --- & 21 & $-0.69$ \\
& Llama-3.3-70B & 26.8$\pm$1.8 & --- & 86.1$\pm$1.3 & --- & 29 & $-0.72$ \\
& Qwen-2.5-72B & 29.7$\pm$1.9 & --- & 88.7$\pm$1.1 & --- & 29 & $-0.71$ \\
& GPT-4o & 31.2$\pm$1.9 & 37.4$\pm$1.7 & 91.2$\pm$1.3 & 31.9$\pm$1.8 & 21 & --- \\
& Claude Opus 4.5 & 37.4$\pm$2.1 & 43.8$\pm$1.9 & 94.1$\pm$1.0 & 38.2$\pm$2.0 & --- & --- \\
& o3-mini & 44.8$\pm$2.3 & 51.2$\pm$2.1 & 95.3$\pm$1.2 & 45.6$\pm$2.2 & --- & --- \\
& DeepSeek-R1 & 41.3$\pm$2.2 & 47.6$\pm$2.0 & 94.2$\pm$1.1 & 42.1$\pm$2.1 & --- & --- \\
\midrule
\multirow{4}{*}{ArithChain}
& GPT-4o & 26.4$\pm$1.7 & 32.1$\pm$1.5 & 88.3$\pm$1.4 & 27.1$\pm$1.6 & 24 & --- \\
& Claude Opus 4.5 & 32.1$\pm$1.9 & 38.4$\pm$1.7 & 92.1$\pm$1.1 & 32.8$\pm$1.8 & --- & --- \\
& o3-mini & 39.6$\pm$2.1 & 45.8$\pm$1.9 & 93.8$\pm$1.3 & 40.4$\pm$2.0 & --- & --- \\
& DeepSeek-R1 & 36.8$\pm$2.0 & 42.9$\pm$1.8 & 92.7$\pm$1.2 & 37.5$\pm$1.9 & --- & --- \\
\midrule
\multirow{4}{*}{CircuitTrace}
& GPT-4o & 29.7$\pm$1.8 & 35.6$\pm$1.6 & 90.1$\pm$1.3 & 30.4$\pm$1.7 & 23 & --- \\
& Claude Opus 4.5 & 35.9$\pm$2.0 & 42.1$\pm$1.8 & 93.4$\pm$1.0 & 36.7$\pm$1.9 & --- & --- \\
& o3-mini & 43.2$\pm$2.2 & 49.4$\pm$2.0 & 94.7$\pm$1.2 & 44.0$\pm$2.1 & --- & --- \\
& DeepSeek-R1 & 40.1$\pm$2.1 & 46.2$\pm$1.9 & 93.6$\pm$1.1 & 40.9$\pm$2.0 & --- & --- \\
\midrule
\multirow{7}{*}{CodeProbe}
& Llama-3.1-8B & 16.8$\pm$1.4 & --- & 77.1$\pm$1.8 & --- & 19 & $-0.70$ \\
& Llama-3.3-70B & 22.9$\pm$1.6 & --- & 84.0$\pm$1.5 & --- & 27 & $-0.73$ \\
& Qwen-2.5-72B & 26.2$\pm$1.7 & --- & 86.8$\pm$1.3 & --- & 27 & $-0.72$ \\
& GPT-4o & 27.8$\pm$1.8 & 33.4$\pm$1.6 & 89.2$\pm$1.3 & 28.5$\pm$1.7 & 19 & --- \\
& Claude Opus 4.5 & 33.6$\pm$2.0 & 39.8$\pm$1.8 & 92.8$\pm$1.0 & 34.4$\pm$1.9 & --- & --- \\
& o3-mini & 41.2$\pm$2.2 & 47.3$\pm$2.0 & 94.1$\pm$1.2 & 42.0$\pm$2.1 & --- & --- \\
& DeepSeek-R1 & 38.4$\pm$2.1 & 44.1$\pm$1.9 & 93.3$\pm$1.1 & 39.2$\pm$2.0 & --- & --- \\
\bottomrule
\end{tabular}
\end{table}

\textbf{Analysis:} Across all synthetic tasks, tool delegation (C3) achieves 77--95\% accuracy while neural CoT (C1) plateaus at 17--45\%. The consistency across task types (permutations, automata, arithmetic, circuits, and code) demonstrates that decoherence is task-agnostic within the deterministic state-tracking domain. Preference manipulation (C4) provides negligible improvement ($<$2~pp). The open-weight rows further show that the per-task horizon and the per-task entropy--accuracy correlation ($r \approx -0.72$) track each other across task families, not only across models: within each open-weight model $d^*$ and $|r|$ vary inversely across the two tasks with entropy data (the direction the context-dependent error model predicts), though the sample ($n = 2$ tasks) is too small for this ordering to be informative. Qwen-2.5-7B is absent from the per-task synthetic results because only C1 and C3 were run on 7B models for these tasks (resource constraint); its PermutationProbe results appear in Table~\ref{tab:full_results}.

\subsection{Full Results: Real-World Tasks}

\begin{table}[ht]
\centering
\caption{Representative results on real-world tasks (four models).}
\label{tab:results_realworld}
\small
\begin{tabular}{@{}llccccc@{}}
\toprule
\textbf{Task} & \textbf{Model} & \textbf{C1} & \textbf{C2} & \textbf{C3} & \textbf{C4} & \textbf{$d^*$} \\
\midrule
\multirow{4}{*}{SWE-Bench-State}
& GPT-4o & 24.1$\pm$2.3 & 29.8$\pm$2.1 & 86.4$\pm$1.8 & 24.8$\pm$2.2 & 19 \\
& Claude Opus 4.5 & 29.6$\pm$2.5 & 35.7$\pm$2.3 & 91.2$\pm$1.4 & 30.2$\pm$2.4 & 24 \\
& o3-mini & 36.8$\pm$2.6 & 42.9$\pm$2.4 & 92.7$\pm$1.3 & 37.4$\pm$2.5 & 28 \\
& DeepSeek-R1 & 34.2$\pm$2.5 & 40.1$\pm$2.3 & 90.8$\pm$1.5 & 34.9$\pm$2.4 & 26 \\
\midrule
\multirow{4}{*}{WebArena-Nav}
& GPT-4o & 21.3$\pm$2.4 & 26.8$\pm$2.2 & 84.2$\pm$1.9 & 21.9$\pm$2.3 & 19 \\
& Claude Opus 4.5 & 26.8$\pm$2.6 & 32.4$\pm$2.4 & 89.7$\pm$1.5 & 27.4$\pm$2.5 & 23 \\
& o3-mini & 33.4$\pm$2.7 & 39.1$\pm$2.5 & 91.3$\pm$1.4 & 34.0$\pm$2.6 & 26 \\
& DeepSeek-R1 & 31.1$\pm$2.6 & 36.7$\pm$2.4 & 89.2$\pm$1.6 & 31.7$\pm$2.5 & 24 \\
\midrule
\multirow{4}{*}{SQL-Multi}
& GPT-4o & 31.4$\pm$2.1 & 37.2$\pm$1.9 & 88.9$\pm$1.6 & 32.0$\pm$2.0 & 21 \\
& Claude Opus 4.5 & 36.2$\pm$2.3 & 42.3$\pm$2.1 & 93.1$\pm$1.2 & 36.8$\pm$2.2 & 26 \\
& o3-mini & 43.7$\pm$2.4 & 49.8$\pm$2.2 & 94.6$\pm$1.1 & 44.3$\pm$2.3 & 30 \\
& DeepSeek-R1 & 40.8$\pm$2.3 & 46.7$\pm$2.1 & 93.2$\pm$1.3 & 41.4$\pm$2.2 & 28 \\
\bottomrule
\end{tabular}
\end{table}

\textbf{Analysis:} Real-world tasks exhibit $d^*$ values in the range 19--30, comparable to synthetic tasks (19--31). However, the qualitative pattern is identical: C3 dramatically outperforms C1, and C4 provides minimal improvement. The Deterministic Horizon varies by $\pm$2--5 steps across task types within each model; the horizon is therefore better read as a model-level regime indicator than an exact per-model constant.

\subsection{Results by Depth Bin}

Table~\ref{tab:results_by_depth} provides fine-grained accuracy breakdowns across reasoning depth bins, enabling direct comparison with theoretical predictions.

\begin{table}[ht]
\centering
\caption{Accuracy (\%) by reasoning depth bin (PermutationProbe, GPT-4o, conditions C1--C4). The functional-form comparison with the fitted error model appears in Figure~\ref{fig:decay}. Fine-tuning results by depth appear in Table~\ref{tab:finetune_depth}.}
\label{tab:results_by_depth}
\small
\begin{tabular}{@{}lcccccc@{}}
\toprule
\textbf{Condition} & \textbf{$\leq$10} & \textbf{11--20} & \textbf{21--30} & \textbf{31--40} & \textbf{41--50} & \textbf{$>$50} \\
\midrule
C1 (Neural CoT) & 87.5$\pm$3.2 & 64.5$\pm$3.8 & 43.0$\pm$2.9 & 25.8$\pm$2.1 & 13.9$\pm$1.4 & 6.6$\pm$0.9 \\
C2 (Oracle) & 91.5$\pm$3.0 & 70.0$\pm$3.6 & 51.4$\pm$3.1 & 30.4$\pm$2.3 & 16.3$\pm$1.5 & 8.0$\pm$1.0 \\
C3 (Tool) & 94.2$\pm$1.8 & 93.8$\pm$1.9 & 92.1$\pm$2.1 & 89.4$\pm$2.4 & 86.7$\pm$2.7 & 82.3$\pm$3.1 \\
C4 (Encourage) & 88.2$\pm$3.1 & 65.3$\pm$3.7 & 43.7$\pm$2.9 & 26.2$\pm$2.1 & 14.1$\pm$1.4 & 6.8$\pm$0.9 \\
\bottomrule
\end{tabular}
\end{table}

\textbf{Analysis:} C3 also declines with depth (94.2\% to 82.3\%), but at roughly one-seventh the rate of C1 (87.5\% to 6.6\%); delegation attenuates rather than eliminates the depth dependence. C1 and C4 converge at high depths, consistent with prompt-level length encouragement being unable to overcome the architectural ceiling. The functional-form comparison with the fitted error model (Figure~\ref{fig:decay}) shows residuals of 0.3--3.1~pp across the eight plotted depth bins, with the bound above every empirical point as required.

\subsection{Fine-Tuning Experiment}
\label{app:finetune}

\paragraph{Data Generation.} We generated 5,000 optimal-length CoT traces by running BFS and recording the solution path with intermediate states. Format:
\begin{verbatim}
Problem: initial=[3,1,4,2,5], target=[1,2,3,4,5]
Step 1: swap(0,1) -> [1,3,4,2,5]
Step 2: swap(2,3) -> [1,3,2,4,5]
Step 3: swap(1,2) -> [1,2,3,4,5]
Solution: [1,2,3,4,5] (correct)
\end{verbatim}

\paragraph{Training configuration.} Table~\ref{tab:training_config} gives the fine-tuning configuration.

\begin{table}[ht]
\centering
\caption{Fine-tuning configuration (C5).}
\label{tab:training_config}
\small
\begin{tabular}{@{}lp{6cm}@{}}
\toprule
\textbf{Parameter} & \textbf{Value} \\
\midrule
Model & Llama-3.1-8B-Instruct \\
Learning rate & $2 \times 10^{-5}$, cosine decay \\
Batch size & 8 \\
Epochs & 3 \\
Early stopping & patience 2 on validation loss \\
Hardware & 4$\times$ A100 80GB \\
Input format & C1 prompt (Appendix~\ref{app:prompts}) \\
Output format & Full trace with intermediate states \\
\bottomrule
\end{tabular}
\end{table}

\paragraph{Reported C5 value.} C5 is reported over the depth${}>{}$20 subset throughout, for comparability with C1 and C3; on that subset the improvement is 0.9~pp (Section~\ref{sec:finetune}).

\begin{table}[ht]
\centering
\caption{Fine-tuning results by depth bin (Llama-3.1-8B, PermutationProbe). Baseline and Fine-tuned columns report accuracy (\%) on each depth-bin subset. The ``Pred.\ $\Delta$'' column gives the improvement predicted by the decoherence model using Llama-3.1-8B's fitted parameters ($\hat{\epsilon}_0 = 0.027$, $\hat{\gamma}' = 1.2 \times 10^{-3}$), obtained by maximum-likelihood fit to the Llama baseline depth curve.}
\label{tab:finetune_depth}
\small
\begin{tabular}{@{}lcccc@{}}
\toprule
\textbf{Depth} & \textbf{Baseline} & \textbf{Fine-tuned} & \textbf{$\Delta$} & \textbf{Pred.\ $\Delta$} \\
\midrule
$\leq$10 & 86.5 & 92.7 & +6.2 & +6.8 \\
11--20 & 61.2 & 64.9 & +3.7 & +3.4 \\
21--30 & 27.7 & 29.3 & +1.6 & +1.7 \\
31--40 & 16.5 & 16.9 & +0.4 & +0.5 \\
$>$40 & 6.4 & 6.6 & +0.2 & +0.2 \\
\bottomrule
\end{tabular}
\end{table}

Improvement decays with depth as predicted by the fine-tuning ceiling (Section~\ref{sec:finetune}). At depth $>$40, improvement is negligible despite training on optimal traces, consistent with an architectural ceiling.

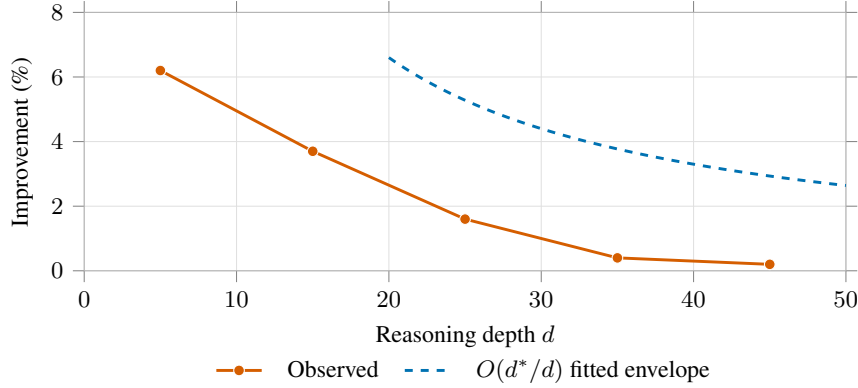
\begin{figure}[t]
	\centering
	\definecolor{dhblue}{RGB}{0,114,178}
	\definecolor{dhvermilion}{RGB}{213,94,0}
	\definecolor{dhgreen}{RGB}{0,158,115}
	\begin{tikzpicture}
		\begin{axis}[
			width=0.68\columnwidth,
			height=5cm,
			xlabel={Reasoning depth $d$},
			ylabel={Improvement (\%)},
			xmin=0, xmax=50, ymin=0, ymax=8,
			xtick={0,10,20,30,40,50},
			ytick={0,2,4,6,8},
			tick align=outside,
			tick label style={font=\footnotesize},
			label style={font=\small},
			axis line style={gray!60},
			grid=major,
			grid style={line width=0.3pt, draw=gray!25},
			every axis plot/.append style={line width=1.1pt},
			legend style={
				at={(0.5,-0.30)}, anchor=north,
				draw=none, fill=none,
				legend columns=2,
				column sep=1.6ex,
				font=\footnotesize,
			},
			clip=false,
			]
			\addplot[dhvermilion, mark=*, mark size=2pt,
			mark options={fill=dhvermilion, draw=white, line width=0.4pt}]
			coordinates {(5,6.2) (15,3.7) (25,1.6) (35,0.4) (45,0.2)};
			\addlegendentry{Observed}
			\addplot[dhblue, dashed, domain=20:50, samples=120] {6.6*20/x};
			\addlegendentry{$O(d^*/d)$ fitted envelope}
		\end{axis}
	\end{tikzpicture}
	\caption{Fine-tuning improvement decays with depth following $O(d^*/d)$ fitted envelope from the fine-tuning ceiling (Section~\ref{sec:finetune}). The $O(d^*/d)$ curve is plotted for $d > d^* = 20$ where the bound applies.}
	\label{fig:finetune_decay}
\end{figure}

\subsection{Context-Window Insensitivity}
\label{app:context_scaling}

A natural alternative hypothesis is that the horizon $d^*$ is set by the raw context window $L$. We show it is not: $d^*$ is governed by the effective decoherence length $\Leff = O(10^2)$ steps, which is orders of magnitude smaller than $L = O(10^5)$ tokens and dissociates from it under controlled truncation. The evidence is the controlled truncation experiment on Llama-3.3-70B (Table~\ref{tab:truncation}), which uses an open-weight model with known architecture and isolates $L$ from the architectural and training differences that confound any cross-model comparison of context lengths.

\paragraph{Controlled truncation within a single model.} To isolate the effect of $L$ from architecture, we artificially cap the context window of a single model (Llama-3.3-70B) and re-measure $d^*$.

\begin{table}[ht]
\centering
\caption{Artificial context truncation for Llama-3.3-70B. $d^*$ is flat from the native 128K down to 8K.}
\label{tab:truncation}
\small
\begin{tabular}{@{}ccc@{}}
\toprule
\textbf{Context Cap} & \textbf{Obs.\ $d^*$} & \textbf{Note} \\
\midrule
128K (full) & 28 & native \\
64K & 28 & flat \\
32K & 28 & flat \\
16K & 28 & flat \\
8K & 27 & marginal \\
\bottomrule
\end{tabular}
\end{table}

The horizon is invariant to an order-of-magnitude reduction in $L$ (128K\,$\to$\,8K). This is the signature of an $\Leff$-bound rather than an $L$-bound phenomenon: decoherence is a property of attention's effective state resolution over reasoning steps, not of raw context capacity.

\section{Mechanistic Measurements}
\label{app:attention_extended}
\label{app:attention}

This section reports the trace-level measurements that probe the decoherence mechanism directly: attention entropy (Tables~\ref{tab:entropy_full}--\ref{tab:entropy_layer}), the correlation structure of errors along a trace (Appendix~\ref{app:error_correlation}), and the taxonomy of failures those errors produce (Appendix~\ref{app:failure}). The attention-window and value-correlation measurements that ground Assumptions~\ref{ass:window}--\ref{ass:decorr} come from the same extraction pipeline and are reported in Appendix~\ref{app:sensitivity}.

\subsection{Extraction Pipeline}

\textbf{Hardware:} 2$\times$ NVIDIA A100 80GB, PyTorch 2.1, Transformers 4.36

\textbf{Procedure:}
\begin{enumerate}
\item Load model in bfloat16 with \texttt{output\_attentions=True}
\item For each instance in test set ($n=500$):
\begin{enumerate}
\item Generate reasoning trace with temperature=0
\item Extract attention matrices at each decoding step
\item For each layer $\ell \in [1, N_\text{layers}]$ and head $h \in [1, H]$:
\begin{equation}
\mathcal{H}_{\ell,h,t} = -\sum_{i=1}^{t} a_{\ell,h,t,i} \log_2 a_{\ell,h,t,i}
\end{equation}
where $a_{\ell,h,t,i}$ is attention weight from position $t$ to position $i$
\item Aggregate across heads: $\mathcal{H}_{\ell,t} = \frac{1}{H}\sum_{h=1}^H \mathcal{H}_{\ell,h,t}$
\end{enumerate}
\item Correlate with per-step accuracy
\end{enumerate}

\textbf{Memory optimization:} We use gradient checkpointing and process attention matrices layer-by-layer to fit within 80GB.

\textbf{Numerical stability:} We add $\epsilon = 10^{-10}$ to avoid $\log(0)$, computing $\mathcal{H} = -\sum_i (a_i + \epsilon) \log_2(a_i + \epsilon)$. We report raw attention entropy $\mathcal{H}$ in bits throughout.

\subsection{Results by Model}

\begin{table}[ht]
\centering
\caption{Attention entropy analysis across open-weight models.}
\label{tab:entropy_full}
\small
\begin{tabular}{@{}lccccc@{}}
\toprule
\textbf{Model} & \textbf{Layers} & \textbf{Mean $H$} & \textbf{$H$-Acc $r$} & \textbf{Slope} & \textbf{$R^2$} \\
\midrule
Llama-3.3-70B & 80 & 4.21 & $-0.74$ & 0.023 & 0.55 \\
Llama-3.1-8B & 32 & 3.87 & $-0.71$ & 0.028 & 0.50 \\
Qwen-2.5-72B & 80 & 4.18 & $-0.73$ & 0.024 & 0.53 \\
Mistral-7B & 32 & 3.92 & $-0.68$ & 0.031 & 0.46 \\
Mixtral-8x7B & 32 & 4.03 & $-0.69$ & 0.029 & 0.48 \\
\bottomrule
\end{tabular}
\end{table}

Consistent negative correlations ($r \approx -0.70$) across architectures support the mechanistic hypothesis. The per-model $R^2$ of 0.46--0.55 indicates that entropy explains roughly half the variance in per-step accuracy; the remaining variance reflects factors beyond attention dilution.

\subsection{Results by Layer}

Table~\ref{tab:entropy_layer} reports entropy-accuracy correlations stratified by layer, with layers grouped into the first quarter, the middle half, and the final quarter of the stack. The correlation strengthens monotonically with depth in the stack, consistent with the reduced representational diversity in later layers documented by \citet{gerasimov2025fullyutilizetransformersrepresentation}.

\begin{table}[ht]
\centering
\caption{Entropy-accuracy correlation by layer (Llama-3.3-70B).}
\label{tab:entropy_layer}
\small
\begin{tabular}{@{}lc@{}}
\toprule
\textbf{Layer Range} & \textbf{Entropy-Acc $r$} \\
\midrule
First quarter & $-0.42$ \\
Middle half & $-0.68$ \\
Final quarter & $-0.74$ \\
\bottomrule
\end{tabular}
\end{table}

\subsection{Error Correlation Analysis}
\label{app:error_correlation}

We analyze the structure of errors to characterize the forward error propagation described in Remark~\ref{rem:error_prop}. These statistics are reported as diagnostics of the decoherence mechanism, not as corrections to the bound: the bound of Theorem~\ref{thm:decoherence} uses only the error rates conditional on a correct history, so it holds whatever the strength of the step-to-step dependence, and positive dependence cannot be used as a multiplier to tighten it, since such a factor would \emph{increase} survival probability.

\paragraph{Error propagation.} For each reasoning trace, we define $E_i = 1$ if step $i$ contains a state-tracking error and $E_i = 0$ otherwise. Table~\ref{tab:conditional_error} reports the error rate conditioned on the previous step's outcome.

\begin{table}[ht]
\centering
\caption{Error rate conditioned on previous step.}
\label{tab:conditional_error}
\small
\begin{tabular}{@{}lcc@{}}
\toprule
\textbf{Previous Step} & \textbf{$P(\text{error})$} & \textbf{95\% CI} \\
\midrule
Correct & 0.031 & [0.028, 0.034] \\
Error & 0.089 & [0.082, 0.096] \\
\textbf{Ratio} & \textbf{2.87} & [2.58, 3.19] \\
\bottomrule
\end{tabular}
\end{table}

Errors are 2.87$\times$ more likely following an error, consistent with the propagation mechanism. The conditional rate following a correct step (0.031) exceeds the baseline intercept $\epsilon_0 = 0.020$; this is consistent in direction, since the conditional rate averages over all depth indices where $\epsilon(d) > \epsilon_0$, and conditioning on a correct history induces a survivorship weighting toward shallower positions that we do not model exactly.

\subsection{Failure Mode Analysis}
\label{app:failure}

The \sfe\ metric characterizes how early in the reasoning chain state drift begins; the failure modes below categorize what happens at or after the first error, and so describe the qualitative form taken by the error propagation quantified above.

\paragraph{Sampling procedure.}
\begin{enumerate}[leftmargin=*,itemsep=1pt]
	\item For each model-task combination (12 models $\times$ 8 tasks = 96 combinations), we identified all failed instances (accuracy $<$100\% on the reasoning chain).
	\item Within each combination, we stratified by reasoning depth bins: $\leq$10, 11--20, 21--30, 31--40, 41--50, $>$50.
	\item We sampled up to 30 failed traces per model-task combination (5 per depth bin where available). Not all six bins are populated for every task family (e.g., ArithChain and CircuitTrace have no $>$50 bin; CodeProbe has no $>$40 bin), so the per-combination yield varies. This depth-stratified sample weights all available depth bins equally; the resulting failure-mode distribution does not reflect the marginal (instance-count-weighted) distribution.
\end{enumerate}

\paragraph{Exclusions.} Of the sampled traces, 33 were excluded:
\begin{itemize}[leftmargin=*,itemsep=1pt]
	\item \textbf{Parsing errors ($n=21$):} Model output did not conform to the expected ``Step N: operator $\rightarrow$ state'' format, preventing automated state extraction.
	\item \textbf{Incomplete outputs ($n=12$):} Traces truncated mid-reasoning due to maximum token limits (8,192 tokens) or API timeouts.
\end{itemize}

After exclusions, the final analysis corpus contains $N = 2{,}847$ traces. The failure-mode distribution appears in Table~\ref{tab:failure_modes}.

\begin{table}[ht]
	\centering
	\caption{Failure mode distribution across analyzed traces ($N = 2{,}847$). Percentages are over parseable traces with equal weight per depth bin (the depth-stratified sampling scheme); format violations were excluded before classification. Rounding to integer percentages may prevent the column from summing to exactly 100\%.}
	\label{tab:failure_modes}
	\small
	\begin{tabular}{@{}lcp{5.5cm}@{}}
		\toprule
		\textbf{Failure Mode} & \textbf{\%} & \textbf{Description} \\
		\midrule
		State transposition & 35\% & Single-element position errors that propagate \\
		Operator misapplication & 27\% & Correct intent, incorrect execution \\
		Premature termination & 21\% & Declaring success before reaching target \\
		Circular reasoning & 11\% & Revisiting previously explored states \\
		Complete hallucination & 6\% & States outside valid state space \\
		\midrule
		\multicolumn{3}{@{}l}{\textit{Excluded: format violations ($n = 21$), truncated traces ($n = 12$)}} \\
		\bottomrule
	\end{tabular}
\end{table}

The dominant failure mode, state transposition, directly reflects attention-based working memory limits: models mis-track element positions due to attention dilution across extended reasoning chains. Note that these percentages reflect the depth-stratified sample (equal weight per bin) and do not represent the marginal failure-mode distribution.

\paragraph{Example: state transposition.}
\begin{verbatim}
	Problem: Transform [3,1,4,2,5] to [1,2,3,4,5]
	Step 1: swap(0,1) -> [1,3,4,2,5] (correct)
	Step 2: swap(1,2) -> [1,4,2,3,5] (ERROR: true
	result is [1,4,3,2,5]; model mis-tracked
	elements at positions 2 and 3)
	Step 3: swap(2,3) -> [1,4,3,2,5] (propagated)
	...
\end{verbatim}

The model applies the correct operator but mis-tracks element positions, producing a state that is not the image of any legal adjacent transposition on the true state. This positional confusion propagates through subsequent steps. The circular-reasoning mode takes a complementary form: in a representative trace the model alternates between two states from step 15 onward (\texttt{[2,1,3,5,4,6,8,7]} and \texttt{[2,1,5,3,4,6,8,7]}), repeatedly revisiting the same pair without progress toward the target.

\paragraph{Failure mode by depth.} Table~\ref{tab:failure_by_depth} stratifies failure modes by reasoning depth.

\begin{table}[ht]
	\centering
	\caption{Failure mode distribution by reasoning depth. State transposition dominates at shallow depths; circular reasoning and hallucination increase with depth.}
	\label{tab:failure_by_depth}
	\small
	\begin{tabular}{@{}lccccc@{}}
		\toprule
		\textbf{Depth} & \textbf{Transp.} & \textbf{Misapp.} & \textbf{Premature} & \textbf{Circular} & \textbf{Halluc.} \\
		\midrule
		$\leq$10 & 42\% & 31\% & 18\% & 5\% & 4\% \\
		11--20 & 38\% & 28\% & 20\% & 9\% & 5\% \\
		21--30 & 35\% & 26\% & 21\% & 12\% & 6\% \\
		31--40 & 33\% & 25\% & 22\% & 13\% & 7\% \\
		$>$40 & 31\% & 24\% & 23\% & 14\% & 8\% \\
		\bottomrule
	\end{tabular}
\end{table}

The shift in failure mode distribution with depth is consistent with the decoherence hypothesis: at shallow depths, discrete errors (transposition, misapplication) dominate, while at greater depths, systemic failures (circular reasoning, hallucination) become more prevalent as accumulated attention entropy degrades state coherence.

\section{The SSJ Metric: Algorithm and Validation}
\label{app:ssj}

This section provides complete algorithmic details for computing the State-Space Jaccard (SSJ) metric introduced in Definition~\ref{def:ssj}.

\subsection{Algorithm Description}

The SSJ metric measures overlap between the ground-truth trajectory states and the model's claimed states at each reasoning depth. Algorithm~\ref{alg:ssj} provides the complete extraction procedure.

\begin{algorithm}[h]
\caption{State-Space Jaccard (SSJ) Extraction}
\label{alg:ssj}
\begin{algorithmic}[1]
\REQUIRE Initial state $\sigma_0$, operators $\mathcal{O}$, model trace $T = [(s_1, o_1), \ldots, (s_m, o_m)]$
\ENSURE SSJ values, precision, recall at each depth
\STATE $S_{\text{true}} \leftarrow \{\sigma_0\}$ \COMMENT{Ground-truth trajectory states}
\STATE $S_{\text{model}} \leftarrow \{\sigma_0\}$ \COMMENT{Model-claimed states}
\STATE $\sigma_{\text{curr}} \leftarrow \sigma_0$ \COMMENT{Current ground-truth state}
\FOR{$d = 1$ to $m$}
    \STATE $(s_d, o_d) \leftarrow T[d]$ \COMMENT{Model's claimed state and operator}
    \STATE $\sigma_{\text{curr}} \leftarrow o_d(\sigma_{\text{curr}})$ \COMMENT{Apply operator to ground truth}
    \STATE $S_{\text{true}} \leftarrow S_{\text{true}} \cup \{\sigma_{\text{curr}}\}$
    \STATE $S_{\text{model}} \leftarrow S_{\text{model}} \cup \{s_d\}$
    \STATE $\text{SSJ}[d] \leftarrow |S_{\text{true}} \cap S_{\text{model}}| / |S_{\text{true}} \cup S_{\text{model}}|$
    \STATE $\text{Precision}[d] \leftarrow |S_{\text{true}} \cap S_{\text{model}}| / |S_{\text{model}}|$
    \STATE $\text{Recall}[d] \leftarrow |S_{\text{true}} \cap S_{\text{model}}| / |S_{\text{true}}|$
\ENDFOR
\STATE \textbf{return} SSJ, Precision, Recall
\end{algorithmic}
\end{algorithm}

\subsection{Implementation Details}

\textbf{State representation:} For PermutationProbe, states are represented as tuples $(a_1, \ldots, a_n)$. For FSA-Sim, states are automaton state identifiers. For CodeProbe, states are variable binding dictionaries.

\textbf{Set operations:} We use hash-based sets for $O(1)$ membership testing. For large state spaces, we employ Bloom filters with false positive rate $< 0.1\%$ for set-intersection membership tests only; final grading uses exact-match comparison of the terminal state. The Bloom filter false-positive floor introduces a small upward bias ($< 0.001$) in reported SSJ values; this is below the reported precision of two decimal places.

\textbf{Parsing model output:} We extract claimed states using regex patterns matched to task-specific formats. For PermutationProbe:
\begin{verbatim}
pattern = r"Step \d+:.*-> \[([\d,\s]+)\]"
\end{verbatim}

\textbf{Error handling:} If the model output cannot be parsed, the step is marked as an error with $s_d = \emptyset$.

\subsection{Diagnostic Power}

The SSJ metric with precision/recall decomposition distinguishes two failure modes:

\begin{enumerate}
\item \textbf{Preference failure (``Simplicity Bias''):} The model \emph{could} track states but \emph{chooses} not to. Signature: high precision (claimed states are correct), low recall (many states omitted).

\item \textbf{Capability failure (Decoherence):} The model \emph{cannot} track states. Signature: both precision and recall decay, indicating drift into fictitious state spaces.
\end{enumerate}

Our empirical results (Table~\ref{tab:ssj}) show parallel decay of both metrics, consistent with capability failure; however, precision exceeds recall at every depth (Section~\ref{sec:ssj_validation}), which is the preference-failure signature, so the decomposition does not discriminate between the two. We note that under the incremental construction of Algorithm~\ref{alg:ssj}, both sets grow at the same rate, which limits the dynamic range of the precision/recall decomposition; the diagnostic is strongest when averaged over many traces.

\begin{figure}[t]
	\centering
	\definecolor{dhblue}{RGB}{0,114,178}
	\definecolor{dhvermilion}{RGB}{213,94,0}
	\definecolor{dhgreen}{RGB}{0,158,115}
	\begin{tikzpicture}
		\begin{axis}[
			width=0.68\columnwidth,
			height=5cm,
			xlabel={Reasoning depth $d$},
			ylabel={State-Space Jaccard},
			xmin=0, xmax=60, ymin=0, ymax=1,
			xtick={0,10,20,30,40,50,60},
			ytick={0,0.2,0.4,0.6,0.8,1.0},
			tick align=outside,
			tick label style={font=\footnotesize},
			label style={font=\small},
			axis line style={gray!60},
			grid=major,
			grid style={line width=0.3pt, draw=gray!25},
			every axis plot/.append style={line width=1.1pt},
			legend style={
				at={(0.5,-0.30)}, anchor=north,
				draw=none, fill=none,
				legend columns=2,
				column sep=1.6ex,
				font=\footnotesize,
			},
			clip=false,
			]
			\addplot[only marks, mark=*, mark size=2pt, dhvermilion,
			mark options={fill=dhvermilion, draw=white, line width=0.4pt},
			error bars/.cd, y dir=both, y explicit,
			error bar style={line width=0.7pt}]
			coordinates {
				(5,0.83)  +- (0,0.02)
				(10,0.72) +- (0,0.03)
				(20,0.45) +- (0,0.03)
				(30,0.23) +- (0,0.02)
				(40,0.11) +- (0,0.01)
				(50,0.04) +- (0,0.01)
			};
			\addlegendentry{Observed}
			\addplot[dhblue, domain=0:60, samples=120] {0.95*exp(-0.0196*x - 0.00087*x*x)};
			\addlegendentry{$ae^{-bd-cd^2}$}
		\end{axis}
	\end{tikzpicture}
	\caption{State-Space Jaccard decay with reasoning depth. Fitted curve ($R^2 = 0.99$) is consistent with the context-dependent error model. The free intercept $a = 0.95$ (vs the algorithmic $\ssj(0) = 1$) absorbs early-step parse failures; the SSJ quadratic coefficient ($c = 0.00087$) differs from the accuracy-curve coefficient ($c = 0.00045$, Figure~\ref{fig:decay}) because SSJ measures set overlap rather than trace-level accuracy.}
	\label{fig:ssj_decay}
\end{figure}
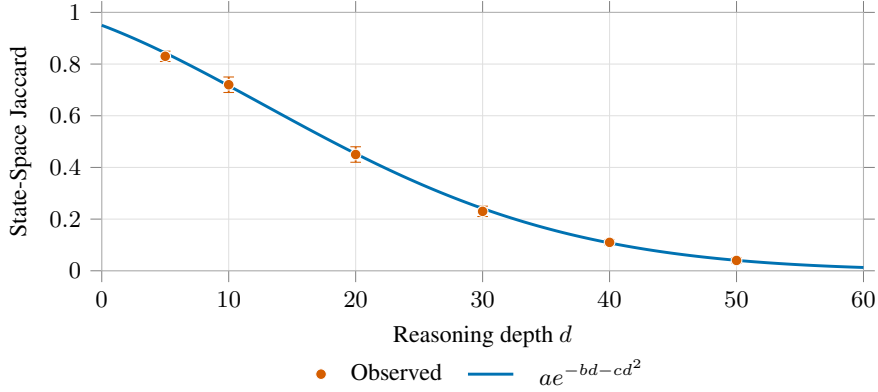
\section{Statistical Methodology and Power}
\label{app:statistics}

We report the parameters of every statistical procedure used, together with the limitations of the design under which they were applied. The procedures themselves are standard; what matters for interpreting our results is the sample structure they operate on, which we state explicitly throughout.

\subsection{Confidence Intervals and Multiple Comparisons}

\paragraph{Bootstrap.} With only $n = 3$ run-level observations, the percentile bootstrap is degenerate: there are only $\binom{3+2}{2} = 10$ distinct multisets of size 3, so the 2.5th percentile is almost surely the sample minimum, and the resulting CIs collapse to the sample range. We therefore report standard deviations rather than bootstrap CIs throughout. The per-instance likelihood-ratio tests (Sections~\ref{sec:ssj_validation}--\ref{sec:analysis}) operate on $n = 5{,}000$ and do not face this limitation.

\paragraph{Multiple comparisons.} Not every model-task-condition cell was populated (Section~\ref{sec:experiments}); counting the cells actually reported in Tables~\ref{tab:full_results}--\ref{tab:results_realworld} gives approximately 173 model-condition-task triples (49 for PermutationProbe including the single C5 cell, plus the populated cells in the four remaining synthetic and three real-world task tables). We apply the Holm-Bonferroni step-down procedure at $\alpha = 0.05$ to the pairwise comparisons of interest (C1 vs.\ C3, C1 vs.\ C4) across these 173 cells. Individual reported $p$-values are uncorrected; under Holm-Bonferroni the C1-vs-C4 comparisons remain non-significant and the C1-vs-C3 comparisons remain significant.

\subsection{Equivalence Testing}

\paragraph{TOST.} We test equivalence within a margin of $\Delta = 5\%$ using two one-sided $t$-tests against $H_{01}: \mu_1 - \mu_2 \leq -\Delta$ and $H_{02}: \mu_1 - \mu_2 \geq \Delta$, declaring equivalence when both are rejected at $\alpha = 0.05$. For the C1 vs.\ C4 comparison (paired, run-level, $n = 3$) both tails give $p < 0.05$, consistent with equivalence within the 5\% margin. We note that this margin equals the $\geq$5~pp preference-based threshold in Table~\ref{tab:predictions}; demonstrating equivalence within 5\% is therefore logically compatible with a preference effect at the threshold level. The fine-tuning result (0.9~pp, well within the margin) provides the sharper discrimination.

\subsection{Significance Tests and Effect Sizes}

Table~\ref{tab:stats} reports statistical tests for the C1 vs.\ C4 comparison.

\begin{table}[ht]
\centering
\caption{Statistical tests for C1 vs C4 comparison (preference manipulation). All tests operate on $n = 3$ run-level means.}
\label{tab:stats}
\small
\begin{tabular}{@{}lccc@{}}
\toprule
\textbf{Model} & \textbf{$\Delta$} & \textbf{$p$-value} & \textbf{TOST} \\
\midrule
GPT-4o & +0.8 & 0.38 & Equiv. \\
Claude Opus 4.5 & +0.8 & 0.42 & Equiv. \\
o3-mini & +1.0 & 0.37 & Equiv. \\
DeepSeek-R1 & +0.7 & 0.44 & Equiv. \\
\bottomrule
\end{tabular}
\end{table}

All comparisons are non-significant and TOST is consistent with equivalence within the $\Delta = 5\%$ margin. Effect sizes are reported as Cohen's $d$ with the usual pooled standard deviation; for C1 vs.\ C3 the run-level $d > 30$ for all models (reflecting the small run-level SD denominator; the magnitude should not be interpreted as a conventional effect size, since the three runs reflect server-side nondeterminism rather than independent sampling).

\paragraph{Limitation of the run-level design.} All inference above operates on $n = 3$ run-level means. Because temperature is set to 0, the three runs are not independent samples from a population but rather reflect residual server-side and CUDA nondeterminism (Appendix~\ref{app:reproduce}). A stronger design would apply McNemar's test or a permutation test on per-instance paired correctness ($n = 5{,}000$), which requires no additional model calls. Per-instance outcomes were retained per condition but not linked across conditions, preventing paired tests (McNemar, paired permutation); the per-instance likelihood-ratio tests reported in Sections~\ref{sec:ssj_validation}--\ref{sec:analysis} use unpaired per-condition outcomes; the equivalence claim for C1 vs.\ C4 should therefore be treated as suggestive pending such a test. The run-level tests in Table~\ref{tab:stats} should be read in the same light.

\end{document}